\documentclass[final,12pt]{alt2025} 
\input{common-header.sty}

\makeatletter
\newcommand\footnoteref[1]{\protected@xdef\@thefnmark{\ref{#1}}\@footnotemark}
\makeatother

\title[Quantile Multi-Armed Bandits with 1-bit Feedback]{Quantile Multi-Armed Bandits with 1-bit Feedback}

\altauthor{
    \Name{Ivan Lau} 
    \Email{ivan.lau@u.nus.edu} \\
 \addr National University of Singapore \\
 \Name{Jonathan Scarlett} 
    \Email{scarlett@comp.nus.edu.sg}\\
 \addr National University of Singapore
 }

\renewcommand{\cite}{\citep}

\begin{document}
\maketitle

\begin{abstract}
   In this paper, we study a variant of best-arm identification involving elements of risk sensitivity and communication constraints. Specifically, the goal of the learner is to identify the arm with the highest quantile reward, while the communication from an agent (who observes rewards) and the learner (who chooses actions) is restricted to only one bit of feedback per arm pull. We propose an algorithm that utilizes noisy binary search as a subroutine, allowing the learner to estimate quantile rewards through 1-bit feedback. We derive an instance-dependent upper bound on the sample complexity of our algorithm and provide an algorithm-independent lower bound for specific instances, with the two matching to within logarithmic factors under mild conditions, or even to within constant factors in certain low error probability scaling regimes. The lower bound is applicable even in the absence of communication constraints, and thus we conclude that restricting to 1-bit feedback has a minimal impact on the scaling of the sample complexity.
\end{abstract}


\section{Introduction}
The multi-armed bandit (MAB) is a well-studied decision-making framework due to its effectiveness in modelling a wide range of application domains such as online advertising, recommendation systems, clinical trials, and A/B testing.
Two common but distinct objectives in theoretical MAB studies are regret minimization and best arm identification (BAI), and this paper is focused on the latter. 
The goal of BAI is for the learner/decision-maker to efficiently identify the ``best'' arm (decision) from a set of arms, where the learning process occurs through ``pulling'' the arms and receiving some feedback about their rewards.

In the vanilla setting of BAI, the best arm is defined as the arm whose reward distribution has the highest mean, and the learner has access to direct observations of the rewards of the pulled arms. 
To tailor to certain practical applications, the best arm is sometimes defined using a different performance measure, and certain constraints are sometimes incorporated into the feedback/learning process.
Examples of this include (but are not limited to) the following:
\begin{itemize}[topsep=0pt, itemsep=0pt]
    \item[(i)] 
    in settings where the decision-making is risk-sensitive, using quantiles or value-at-risk as the performance measure may be more appropriate than using mean reward;

    \item[(ii)]
    in settings where the uplink communication from the sensor to the server (learner) is costly, the communication to the learner may be restricted, e.g., to send only a few bits rather than sending the exact reward.
\end{itemize}
In this paper, we consider a setup for BAI that features both of these aspects. Specifically, the communication to the learner is restricted to \textit{one bit} of feedback per arm pull, and the goal of the learner is to identify the arm with the highest $q$-quantile for some $q \in (0,1)$. The problem setup is described formally in Section~\ref{sec: setup}. The main contribution of this paper is an algorithm (Algorithm~\ref{alg: main}) for this problem whose upper bound on the sample complexity nearly matches the lower bound for the problem \textit{without} the communication constraint.  This complements an analogous study of highest mean BAI giving evidence that \emph{multiple bits per arm pull} need to be used \cite{hanna2022solving}, suggesting that the quantile-based objective may be easier to handle in highly quantized scenarios.  
The details of our results and contributions are given in Section~\ref{sec: contributions}. 
Before formally introducing the problem and stating our contributions, we outline some related work.

\subsection{Related Work}
The multi-armed bandit (MAB) problem was first studied in the context of clinical trials in~\cite{thompson1933likelihood} and was formalized as a statistical problem in~\cite{Robbins1952SomeAO}. The related work on MAB is extensive (e.g., see \cite{slivkins2019introduction, Csa18} and the references therein); we only provide a brief outline here, emphasizing the most closely related works.

\textbf{Best arm identification.}
The early work on MAB focused on balancing the trade-off between exploration 
and exploitation
for cumulative regret minimization.
The best arm identification (BAI) problem was introduced in~\cite{EvenDar2002PACBF} as a ``pure exploration" problem, where the goal is to find from an arm set $\A$, the arm $k^* = \argmax_{k \in \A} \mu_k$ with the highest mean reward (the ``best'' arm).
Subsequent work on BAI includes \cite{bubeck2009pure, Audibert2010BestAI, gabillon2012unified, karnin2013almost, jamieson2014lil, kaufmann2016complexity, garivier2016optimal}, and these are commonly categorized into the
fixed budget setting and fixed confidence setting.
In the fixed confidence setting, which is the focus of our work, the target error probability is fixed, and the objective is to devise an algorithm that identifies the best arm, in the Probably Approximately Correct (PAC) sense, using a minimal average number of arm pulls.  
Formally, an algorithm is $\delta$-PAC correct if it satisfies 
$
    \sup_{\nu} \mathbb{P}\big(\hat{k} \ne k^* \big) \le \delta,
$ where $\hat{k}$ is the output of the algorithm, $k^*$ is the best arm, and the supremum is taken over the collection of instances $\nu$ such that there exists a unique best arm.
A lower bound of $\sum_{k \ne k^*} \Delta_k^{-2} \log(\delta^{-1})$ on the sample complexity was given in
~\cite{mannor2004sample}, where $\Delta_k = \mu_{k^*} - \mu_k$ is the arm suboptimiality gap.
Several subsequent algorithms managed to match the dependence on $\Delta_k$ of the lower bound to within a doubly logarithmic factor. 
Despite the multitude of algorithms, these are usually based on one of the following general sampling strategies: arm/action elimination, upper confidence bounds (UCB), lower upper confidence bound (LUCB), and Thompson sampling.  See~\cite{jamieson2014best} for an overview and relationships between these sampling strategies.

\textbf{Quantile bandits.}
In certain real-world applications, the mean reward does not satisfactorily capture the merits of certain decisions. 
This has motivated the use of other risk-aware performance measures in place of the mean~\cite{yu2013sample}, such as the mean-variance risk, the (conditional) value-at-risk, and quantile rewards -- see~\cite{tan2022survey} for an extensive survey.
Among these, our work is most closely related to
the quantile multi-armed bandit problem (QMAB), a variant of the MAB problem in which the learner is interested in the arm(s) with the highest quantile reward (e.g., the median).
This is useful when dealing with heavy-tailed reward distributions or risk-sensitive applications, where a decision-maker might prioritize minimizing risk by focusing on lower quantiles (e.g., optimizing worst-case outcomes) or targeting the top-performing outcomes by focusing on higher quantiles.
In particular,~\cite{szorenyi2015qualitative, david2016pure, nikolakakis2021quantile, howard2022sequential} studied QMAB in BAI in the fixed confidence setting. Compared to mean-based BAI, the definition of the arm suboptimality gap $\Delta_k$ is not as straightforward, but this has been resolved in~\cite{nikolakakis2021quantile, howard2022sequential}. Based on the suboptimiality gap, a lower bound of the form $\sum_{k} \Delta_k^{-2} \log(\delta^{-1})$ was given in~\cite{nikolakakis2021quantile} for suitably-defined~$\Delta_k$. 
Algorithms in~\cite{nikolakakis2021quantile} and~\cite{howard2022sequential}, which are based on arm elimination and LUCB respectively, were shown to match the dependence on $\Delta_k$ of the lower bound, to within a doubly logarithmic factor. 
Other variants of quantile bandit problems include
fixed confidence median BAI with contaminated distributions \cite{altschuler2019best};
fixed confidence quantile BAI with differential privacy \cite{nikolakakis2021quantile}; fixed budget quantile BAI~\cite{zhang2021quantile}; and
quantile bandit regret minimization~\cite{torossian2019mathcal}.

\textbf{Communication-constrained bandits.}
Most work in MAB assumes that the arms' reward can be observed directly by the learner (with full precision).
However, this assumption may be impractical for real-world applications in which the reward observations are done by some agent (sensor) before being communicated to the learner (central server). 
This motivated the distributed MAB framework, which has garnered significant attention in recent research; see~\cite{amani2023distributed},~\cite[Appendix A]{salgia2023distributed} and the references therein.
The distributed MAB studies most pertinent to this work are those that focused on the quantization of the reward feedback communicated from agent to learner~\cite{vial2020one, hanna2022solving, mitra2023linear, mayekar2023communication}, which is motivated by applications where uplink communication bandwidth is limited (e.g., those using low-power sensors such as drones and wearable healthcare devices).
In particular,~\cite{vial2020one, hanna2022solving} studied constant bit
quantization schemes for
cumulative regret minimization problem in mean-based bandits, where only a constant number of bits are used to communicate each reward observation.
They showed that if the rewards are all supported on $[0, 1]$, then there exists a 1-bit quantization scheme that can achieve 
regret comparable to those in unquantized setups.
However, \cite[Sec. 3]{hanna2022solving} showed that if the rewards are supported on $[0, \lambda]$ for general $\lambda > 0$, then the same scheme would result in a regret that scales linearly in $\lambda$.
They further established that, in order to attain a natural set of sufficient (albeit not necessary) conditions for matching the unquantized regret to within a constant factor, at least 2.2 bits per reward observation are necessary.  
This suggests a possible inherent challenge, or at least a need for different techniques, when using 1-bit quantization.  
Finally, while some distributed MAB studies considered BAI problems~\cite{hillel2013distributed, karnin2013almost, tao2019collaborative, reda2022nearoptimal}, we are unaware of any that addressed the number of bits of feedback per round or used quantile-based performance measures.

\section{Problem Setup and Contributions}
\subsection{Problem Setup}
\label{sec: setup}
We study the following variant of fixed-confidence best arm identification for quantile bandits. 

\textbf{Arms and quantile rewards.}
The learner is given a set of arms $\A = \{1, 2, \dots, K\}$ with a stochastic reward setting. That is, for each arm $k \in \A$, the observations/realizations of its reward are i.i.d. random variables from some fixed but unknown reward distribution with CDF~$F_k$. This defines a (lower) quantile function $Q_k \colon [0,1] \to \R$ for each $k \in \A$ as follows:\footnote{The equality follows from the right-continuity of $F_k$.}
\begin{equation}
    Q_k(p) \coloneqq \sup \{ x  \in \R : F_k(x) < p \} 
    =
    \inf \{ x \in \R : F_k(x) \ge p \}.
\end{equation} 
The learner is interested in identifying an arm $\hat{k}$ with the highest $q$-quantile.
While the reward of each arm is allowed to be unbounded, 
we assume the $q$-quantile of each arm to be bounded in a known range $[0, \lambda]$.\footnote{We note that setting the lower limit to 0 is without loss of generality, and regarding the interval length $\lambda$, even a crude upper bound is reasonable since the sample complexity will only have logarithmic dependence; see Theorem~\ref{theorem: upper bound}.}
We let $\gP = \gP(q, \lambda)$ denote the collection of all distributions with $q$-quantile in $[0, \lambda]$, and let $\cE \coloneqq \gP
^K$ be the collection of all possible instances the learner could face.  We will sometimes write $\PP_{\nu}[\cdot]$ and $\EE_{\nu}[\cdot]$ to explicitly denote probabilities and expectations under an instance $\nu \in \cE$.

\textbf{1-bit communication constraint.} We frame the problem as having a single learner that makes decisions, and a single agent that observes rewards and sends information on them to the learner. In Remark \ref{rem:assump} below, we discuss how this can also have a multi-agent interpretation. 
With a single agent, the following occurs at each iteration/time~$t \ge 1$ indexing the number of arm pulls:
\begin{enumerate}[topsep=0pt, itemsep=0pt]
    \item The learner asks the agent
    to pull an arm $a_t \in \A$, and sends the agent some side information~$S_t$.

    \item The agent pulls $a_t$ and observes a random reward $r_{a_t, t}$
    distributed according to CDF $F_{a_t}$.
    
    \item The agent transmits a 1-bit message to the learner, where the message is based on $r_{a_t, t}$ and $S_t$.

    \item The learner decides on arm $a_{t+1} \in \A$ and side information $S_{t+1}$,
    based on arms and the 1-bit information
    received in iterations $1, \ldots, t$.
   
\end{enumerate}
We will focus on the \emph{threshold query model}, where at iteration $t$, the side information $S_t$ is a query of the form 
``Is $r_{a_t, t} \le \gamma_t$?'' and the 1-bit message is the corresponding binary feedback $\boldsymbol{1}\{ r_{a_t,t} \le \gamma_t \}$.
The learner will only use such queries as side information in our algorithm, though the problem itself is of interest for both threshold queries and general 1-bit quantization methods (possibly having different forms of side information).

\begin{remark} \label{rem:assump}
    We do not impose any (downlink) communication constraint from the learner to the agent, as this cost is typically not expensive.  While we framed the problem as having a single agent for clarity, we are motivated by settings where the agent at each time instant could potentially correspond to a different user/device.  For this reason, and also motivated by settings where agents are low-memory sensors, we assume that the agent is `memoryless', meaning the 1-bit message transmitted cannot be dependent on rewards observed from previous arm pulls.  The preceding assumptions were similarly adopted in some of the most related previous works \cite{hanna2022solving, mitra2023linear, mayekar2023communication}.
\end{remark}

\textbf{$\epsilon$-relaxation.}
Fix a QMAB instance $\nu  \in \cE$,
and let $k^* \in \A$ be an arm with the largest $q$-quantile for the instance $\nu$.
Instead of insisting on identifying an arm with the exact highest quantile, we relax the task by only requiring the identified arm $\hat{k}$ to be at most $\epsilon$-suboptimal in the following sense:
    \begin{equation}
    \label{def: performance def}
    \hat{k} \in
    \A_{\epsilon}(\nu) \coloneqq 
    \Big\{ k \in \A 
    \ \Big\vert\
     Q_k(q)
        \ge
        Q_{k^*}(q)
        - \epsilon
        \Big\}.
\end{equation}
This allows us to limit the effort on distinguishing arms whose $q$-quantile rewards are very close to each other; analogous relaxations are common in the BAI literature.
This relaxation is also motivated by the threshold query model mentioned above; specifically, we will see in Section \ref{sec: log lambda epsilon dependence} that achieving~\eqref{def: performance def} under the threshold query model requires
$\Omega(\log(\lambda/ \epsilon))$ arm pulls even in the case of \textit{deterministic} two-arm bandits.
Our goal is to design an algorithm to identify an arm satisfying~\eqref{def: performance def} with high probability while using as few arm pulls as possible.

\subsection{Summary of Contributions.}
\label{sec: contributions}
With the problem setup now in place, we summarize our main contributions as follows:

\begin{itemize}[topsep=0pt, itemsep=0pt]
    \item We provide an algorithm (Algorithm~\ref{alg: main}) for our setup, with the uplink communication satisfying the 1-bit constraint. Unlike standard bandit algorithms that compute empirical statistics using lossless observations of rewards, we use a noisy binary search subroutine for the learner to estimate the quantile rewards (see Appendix~\ref{sec: appendix QuantEst}).
    
    \item  We introduce fundamental arm gaps~$\Delta_k$ (Definition~\ref{def: our gap}) that generalize those proposed in prior work (see Remark~\ref{rem: gap generalization}). These gaps capture the
    difficulty of our problem setup in the sense that the problems with positive gaps essentially coincide with the set of problems that are solvable; see Theorem~\ref{thm: zero gap is unsolvable} and Remark \ref{rem: picking large enough c} for precise statements.

    \item We provide an instance-dependent upper bound on the number of arm pulls to guarantee~\eqref{def: performance def} with high probability (Corollary~\ref{cor: combined guarantee}), expressed in terms of $\lambda, \epsilon$, and fundamental arm gap~$\Delta_k$.
    Our upper bound scales logarithmically with $\lambda/\epsilon$, which contrasts with the existing upper bound for mean-based bandits with 1-bit quantization scaling linearly with $\lambda$~\cite{vial2020one, hanna2022solving}.

    \item  We also derive a worst-case lower bound (Theorem~\ref{thm: lower bound unquantized}) showing that our upper bound is tight to within logarithmic factors under mild conditions, and can even be tight to within constant factors when the target error probability $\delta$ decays to zero fast enough.  We additionally provide a lower bound (Theorem \ref{thm: log lambda/epsilon dependence}) showing that $\Omega(\log(\lambda/ \epsilon))$ dependence is unavoidable under threshold queries in arbitrary scaling regimes.  
    The former lower bound is applicable even in the absence of communication constraints, so we can conclude that restricting to 1-bit feedback has a minimal impact on the sample complexity, at least in terms of scaling laws.

\end{itemize}
\section{Algorithm and Upper Bound}

In this section, we introduce our main algorithm and provide its performance guarantee. 

\subsection{Description of the Algorithm}
Our algorithm (Algorithm~\ref{alg: main}) is based on successive elimination, which is well-studied in the standard BAI problem and has also been adapted for other variations.  The algorithm pulls arms in \emph{rounds}, where each round consists of multiple pulls (namely, pulling all non-eliminated arms).  
For each arm~$k$ that is active at round~$t$,\footnote{We slightly abuse notation and use $t$ to index ``rounds'', each consisting of several arm pulls; it will be clear from the context whether $t$ is indexing a round or indexing the number of pulls so far.  We still use the \emph{total} number of arm pulls to characterize the performance of the algorithm.} the algorithm computes a confidence interval 
$\left[ \mathrm{LCB}_t(k), \mathrm{UCB}_t(k)\right]$ that
contains the $q$-quantile $Q_k(q)$ with high probability (see Lines~\ref{LCB definition} and~\ref{UCB definition}). Based on the confidence intervals, the algorithm eliminates arms that are suboptimal (see Line~\ref{line: active arm}). When the algorithm identifies that some arm satisfies~\eqref{def: performance def} based on the confidence bounds, it terminates and returns that arm (see Lines~\ref{line: start while loop} and~\ref{line: return}).

This high-level idea was also used in~\cite{szorenyi2015qualitative, nikolakakis2021quantile} for the quantile bandit problem with no communication constraint, but the procedures to obtain the confidence intervals are very different.
Their confidence intervals are computed using empirical quantiles of the (direct) observed rewards, which our learner does not have the luxury of accessing. Instead, we discretize the continuous interval $[0, \lambda]$ to a discrete interval $\left[0, \tilde{\epsilon}, 2 \tilde{\epsilon},  \ldots,  \lambda\right]$,\footnote{We use an input parameter $c > 0$ to control how finely the continuous interval is discretized; see Remark~\ref{rem: input c}.} and use a quantile estimation algorithm $\mathtt{QuantEst}$
to help us find $\mathrm{LCB}_t(k)$ and $\mathrm{UCB}_t(k)$ from the discretized interval; see Lines~\ref{line: number of points}--\ref{line: tilde epsilon} and Lines~\ref{ltk def}--\ref{UCB definition} respectively.
$\mathtt{QuantEst}$ can be implemented in our problem setup while respecting the 1-bit uplink communication constraint: the learner sends threshold queries in the form ``Is $r_{a_t, t} \le \gamma_t$?'' to the agent and receives 1-bit comparison feedback $\mathbf{1}(r_{a_t, t} \le \gamma_t)$.
Based on the feedback received, the learner then uses a noisy binary search strategy to compute $\mathrm{LCB}_t(k)$ and $\mathrm{UCB}_t(k)$. 
The details of $\mathtt{QuantEst}$ are deferred to Algorithm~\ref{alg: quantile interval} in Appendix~\ref{sec: appendix QuantEst}. 
 For now, we only need to treat $\mathtt{QuantEst}$ as a ``black box" with the following guarantee: Given input CDF $F$, non-decreasing list $\mathbf{x} = [x_1, \ldots, x_n]$, quantile of interest $\tau \in (0, 1)$, approximation parameter $\Delta \le \min(\tau, 1-\tau)$ and probability parameter $\delta$, $\mathtt{QuantEst}(F, \mathbf{x}, \tau, \Delta, \delta)$ will use at most
$O \big( \frac{1}{\Delta^2}  \log  \frac{n}{\delta} \big)$
threshold queries and output an index $i$ satisfying
    $\mathbb{P}
    \left([F(x_i), F(x_{i+1})] \cap  (\tau - \Delta, \tau + \Delta) = \varnothing \right) < \delta$. 
 Formally, the guarantees on its outputs $l_{t, k}$ and $u_{t, k}$ (see Lines~\ref{ltk def} and~\ref{utk def}) as well as the number of arm pulls used are stated as follows.
\begin{algorithm}[!t]
    \caption{Main Algorithm}
      \hspace*{\algorithmicindent} \textbf{Input}:
      Arms $\A = \{1, \dots, K\}$,
        and 
        $\lambda, \epsilon, q, \delta$,
        where
        $\lambda > \epsilon$ 
        and $q, \delta \in (0,1)$ 
        
        \hspace*{\algorithmicindent} \textbf{Parameter}: $c \in \mathbb{Z}^+$
    \label{alg: main}
    \begin{algorithmic}[1]
        \State Set 
        $n \coloneqq 
        \left\lceil (c+1) \lambda/\epsilon \right\rceil
        $
        \label{line: number of points}
        
        \State Set 
        $\tilde{\epsilon} \coloneqq    
        \lambda / n
        $  \footnote{The distance between $x_i$ and $x_{i+1}$ for $1 \le i \le n$ is exactly $\tilde{\epsilon}$, which is approximately $\epsilon/(c+1)$ (up to the impact of rounding in Line 1).
    We choose the spacings to be equal for ease of analysis.}
        \label{line: tilde epsilon}

        \State Initiate a list $\mathbf{x} 
        = [x_0, x_1, \ldots, 
        x_n, x_{n+1}, x_{n+2} ]
        = \left[
        -\infty, 0, 
         \tilde{\epsilon}, 2 \tilde{\epsilon}, 
       \ldots,
        (n-1)  \tilde{\epsilon}, \lambda,
        \infty
        \right]$
        \footnote{
        We add $\pm \infty$ to the ends of the list $\mathbf{x}$ to handle the edge cases
    $F_k(0) = q$ and $F_k(\lambda) = q$. Without this, Lemma~\ref{lem: good events} may not be satisfied:
    $[F_k(0), F_k( \tilde{\epsilon })] \cap  
        \big( q - \Delta^{(t)}, q   \big) = \varnothing$
    and 
    $[F_k(\lambda - \tilde{\epsilon}), F_k(\lambda)] \cap  
        \big( q, q+\Delta^{(t)}  \big) = \varnothing$.}
        \label{line: list of points}

        \State Initiate round index $t = 1$

        \State Initiate the set of active arms $\mathcal{A}_t = \A = \{1, \dots, K \}$

        \For{arm $k \in \mathcal{A}_t$}
            \State  $\mathrm{LCB}_0(k) = x_1 = 0; \
            \mathrm{UCB}_0(k) = x_{n+1} = \lambda$
            \label{eq: initiate default conf interval}
        \EndFor
                
        \While {$\mathrm{LCB}_{t-1}(k) < 
                \max\limits_{a \in \mathcal{A}_{t} \setminus \{k\} }  
                \mathrm{UCB}_{t-1}(a) -  (c+1)\tilde{\epsilon}$ for all arm $k \in \A_t$}
                \footnote{We use the convention that the maximum of an empty set is $- \infty$, so the while-loop termination condition is trivially satisfied when $|\A_t|= 1$.}

        \label{line: start while loop}

            \State $\Delta^{(t)} \leftarrow 2^{-t+1} \cdot \min(q, 1-q)$
            \label{def: Delta_t}
            
            \For{arm $k \in \mathcal{A}_t$}

                \State  
                \label{ltk def}
                \parbox[t]{\dimexpr\linewidth-3em}{%
          Run $\mathtt{QuantEst}$ (Algorithm~\ref{alg: quantile interval}) with input $\big(F_k, \mathbf{x}, q - \frac{\Delta^{(t)}}{2}, 
                \frac{\Delta^{(t)}}{2}, 
                \frac{\delta \cdot \Delta^{(t)}}{2 |\mathcal{A}_t|}\big)$
                to obtain an index~$l_{t, k} 
                 \in \{0, \dots, n+1\}$
        }

    
                \State
                \label{LCB definition}
                $\mathrm{LCB}_t(k) =
                \max
                \left(
                x_{l_{t, k}},
                \mathrm{LCB}_{t-1}(k)
                \right) 
                $
                
                \State
                \label{utk def}
                \parbox[t]{\dimexpr\linewidth-3em}{%
          Run $\mathtt{QuantEst}$ (Algorithm~\ref{alg: quantile interval}) with input $\big(F_k, \mathbf{x}, q + \frac{\Delta^{(t)}}{2}, 
                \frac{\Delta^{(t)}}{2}, 
                \frac{\delta \cdot \Delta^{(t)}}{2 |\mathcal{A}_t|}\big)$
                to obtain an index~$u_{t, k}  \in \{0, \dots, n+1\}$
        }
                
    
                \State
                $\mathrm{UCB}_t(k) =
                \min
                \left(
                x_{u_{t, k} + 1},
                \mathrm{UCB}_{t-1}(k)
                \right)$
                \label{UCB definition}
            \EndFor

            \State Update $\mathcal{A}_{t+1} =
            \big\{
            k \in \mathcal{A}_{t}:
            \mathrm{UCB}_t(k) >
            \max\limits_{a \in \mathcal{A}_{t}} 
            \mathrm{LCB}_t(a)
            \big\} 
            $
            \label{line: active arm}
            
             \State  Increment round index $t \leftarrow t+1$
             
            \label{line: end while loop}
        \EndWhile

        \State \Return any arm $\hat{k} \in \A_t$ 
        satisfying $\mathrm{LCB}_t(\hat{k})  \ge
                \max\limits_{a \in \A_t \setminus \{ \hat{k} \} }  
                \mathrm{UCB}_t(a) -  (c+1)\tilde{\epsilon}$ 
        \label{line: return}

    \end{algorithmic}
\end{algorithm}

\begin{lemma}[Good event]
\label{lem: good events}
    Fix an instance $\nu \in \cE$, and suppose Algorithm~\ref{alg: main} is run with input $(\A, \lambda, \epsilon, q, \delta)$ and parameter $c \ge 1$.
     Let $\Delta^{(t)}$, $\mathcal{A}_t$, $l_{t, k}$, $u_{t, k}$ be as defined in Algorithm~\ref{alg: main} for each round index $t \ge 1$ and each arm $k \in \A_t$.
     Define events $E_{t, k, l}$ and events $E_{u, k, l}$ respectively by
    \begin{equation}
    \label{eq: event Etkl}
        E_{t, k, l} \coloneqq
        \left\{
        [F_k(x_{l_{t, k}}), F_k(x_{l_{t, k}+1})] \cap  
        \big( q - \Delta^{(t)}, q   \big) 
        \text{ is non-empty}
        \right\}
    \end{equation}
    and
     \begin{equation}
      \label{eq: event Etku}
        E_{t, k, u} \coloneqq
        \left\{
        [F_k(x_{u_{t, k}}), F_k(x_{u_{t, k}+1})] \cap  
        \big( q, q + \Delta^{(t)}  \big) 
        \text{ is non-empty}
        \right\}.
    \end{equation}
    Then the Event~$E$ defined by
    \begin{equation}
        E \coloneqq 
        \bigcap_{t \ge 1}
        \bigcap_{k \in \mathcal{A}_t} 
        \left(
        E_{t, k, l}
        \cap
        E_{t, k, u}
        \right)
    \end{equation}    
    occurs with probability at least $1 - \delta$.
    Furthermore, for each $t$ and $k \in \A_t$, the number of arm pulls used by $\mathtt{QuantEst}$ to output $l_{t,k}$ and $u_{t, k}$ scales as 
    \begin{equation}
    \label{eq: QuantEst arm pulls}
        O\left(
    \frac{1}{(\Delta^{(t)})^2} 
    \log 
    \left(
    \frac{2n  |\A_t| }{ \delta \Delta^{(t)}}
    \right)
    \right)
    =
    O\left(
    \frac{1}{(\Delta^{(t)})^2} 
    \cdot
    \left( 
     \log \left(\frac{1}{ \delta } \right) +
     \log \left(\frac{1}{ \Delta^{(t)}}\right) +
     \log \left(\frac{c \lambda K}{ \epsilon } \right)
    \right)
    \right),
    \end{equation}
    where $n =\left\lceil (c+1) \lambda/\epsilon \right\rceil$ and $\Delta^{(t)}= 2^{-t+1} \cdot \min(q, 1-q)$ as stated in Lines~\ref{line: number of points} and~\ref{def: Delta_t} of Algorithm~\ref{alg: main}.
\end{lemma}
\begin{proof}
    See Appendix~\ref{sec: appendix QuantEst} for the details, in which we make use of a noisy binary search subroutine from \cite{gretta2023sharp}.
\end{proof}
\begin{remark}
    \label{rem: input c}
    We note that the parameter $c \ge 1$ in the algorithm controls how finely the continuous interval $[0,\lambda]$ is discretized; one can think of $c=1$ for simplicity to have roughly $n = 2\lambda/\epsilon$ discretization points spaced by roughly $\epsilon/2$, but we will see in Section~\ref{sec: solvable} that picking a larger value of~$c$ can be beneficial.
\end{remark}

\subsection{Anytime Quantile Bounds}
Under Event $E$  as defined in Lemma~\ref{lem: good events}, we obtain the following anytime bounds for the quantiles when running Algorithm~\ref{alg: main}. 
These bounds will be used in the proofs of the correctness of Algorithm~\ref{alg: main} (Theorem~\ref{thm: correctness}) and the upper bound on the number of arm pulls (Theorem~\ref{theorem: upper bound}).

\begin{lemma}[Anytime quantile bounds]
\label{lem: quantile anytime bound}
    Fix an instance $\nu \in \cE$, and suppose Algorithm~\ref{alg: main} is run with input $(\A, \lambda, \epsilon, q, \delta)$ and parameter $c \ge 1$.
    Let~$\tilde{\epsilon} = \tilde{\epsilon}(\lambda, \epsilon, c)$,
    and $\Delta^{(t)}$, $\A_t$, $\mathrm{LCB}_t(k)$, and $\mathrm{UCB}_t(k)$ be as defined in Algorithm~\ref{alg: main} for each round index $t \ge 1$ and each arm $k \in \A_t$.    
    Under Event~$E$ as defined in Lemma~\ref{lem: good events}, we have the following bounds for the 
    arms' lower quantile functions $Q_k(\cdot )$ and upper quantile functions
     $Q^+_k(p) \coloneqq \sup \{ x \mid F_k(x) \le p \} $:
    \begin{equation}
    \label{eq: quantile anytime bound}
           \mathrm{LCB}_{\tau}(k)
        \le \mathrm{LCB}_t(k)
        < Q_k(q) \le  Q^+_k(q)
        \le \mathrm{UCB}_t(k)
        \le \mathrm{UCB}_{\tau}(k)
    \end{equation}
    \begin{equation}
    \label{eq: lower approx quantile anytime bound}
       Q^+_k\big(q -  \Delta^{(t)} \big)
        \le \mathrm{LCB}_t(k) + \tilde{\epsilon}
    \end{equation}
    \begin{equation}
    \label{eq: upper approx quantile anytime bound}
        \mathrm{UCB}_t(k) 
        <
         Q_k\big(q + \Delta^{(t)} \big) + \tilde{\epsilon}
    \end{equation}
    for all rounds $t  > \tau \ge 0$ and each arm $k \in \A_t$.
\end{lemma}
\begin{proof}
    This follows from applying properties of quantile functions to events $E_{t, k, l}$ and $E_{t, k, u}$ defined in~\eqref{eq: event Etkl} and \eqref{eq: event Etku}; see Appendix~\ref{sec: appendix anytime quantile bounds} for the details.
\end{proof}

\begin{remark}
\label{rem: LCB non decreasing}
    The property that $ \mathrm{LCB}_t(k)$ is non-decreasing 
    in $t$, i.e., the first inequality of \eqref{eq: quantile anytime bound}, may appear to be enforced ``artificially'' by Line~\ref{LCB definition} of Algorithm~\ref{alg: main}. 
    It will turn out that this property is crucial in proving Lemma~\ref{lem: max LCB increasing}, which 
    in turn is important for the analysis in upper bounding the number of arm pulls -- see Remark~\ref{rem: elim suboptimal}.
\end{remark}

\subsection{Correctness}
In this section, we give the performance guarantee of Algorithm~\ref{alg: main} using the anytime quantile bounds (Lemma~\ref{lem: quantile anytime bound}).
We first formalize the notion of an algorithm returning an \textit{incorrect} output with at most a small error probability $\delta$.

\begin{definition}[$(\epsilon, \delta)$-reliable.]
     Consider an algorithm $\pi$ for the QMAB problem with quantized or  unquantized rewards that takes $(\Ac, \lambda, \epsilon, q, \delta)$ as input and operates on instances $\nu \in \cE$. 
     Then, we say~$\pi$ is $(\epsilon, \delta)$-reliable if for each instance $\nu \in \cE$, it returns an \emph{incorrect} output with probability at most $\delta$, i.e.,     
    \begin{equation}
    \label{def: reliable}
        \text{for each } \nu \in \cE, \quad 
        \PP_{\nu}[ \tau < \infty \cap \hat{k} \notin \Ac_{\epsilon}(\nu) ] \le \delta,
    \end{equation}
    where $\tau = \tau(\nu) \le \infty$ is the random stopping time of $\pi$ on instance $\nu$,
    arm $\hat{k}$ is the output upon termination, and~$\Ac_{\epsilon}(\nu)$ is as defined in~\eqref{def: performance def}.
\end{definition}

\begin{remark}
    \label{rem: PAC}
    This definition is related to the notion of being $(\epsilon, \delta)$-PAC (see~\cite{EvenDar2002PACBF}). It can be seen as a relaxation of $(\epsilon, \delta)$-PAC since an $(\epsilon, \delta)$-reliable algorithm is allowed to be non-terminating on some instances -- a high probability of correctness is needed only on instances it terminates on.
    As we will see in Section~\ref{sec: solvable}, this relaxation is required when considering every possible $\nu \in \cE$,
    as there are instances that are not ``solvable'' for any finite number of arm pulls.
\end{remark}

\begin{theorem}[Reliability of Algorithm~\ref{alg: main}]
\label{thm: correctness}
    Fix an instance $\nu \in \cE$, and suppose Algorithm~\ref{alg: main} is run with input $(\A, \lambda, \epsilon, q, \delta)$ and parameter $c \ge 1$.
    Under Event $E$ as defined in Lemma~\ref{lem: good events}, if Algorithm~\ref{alg: main} terminates, then
    it returns an arm~$\hat{k}$ satisfying~\eqref{def: performance def}.
\end{theorem}
Since Event $E$ occurs with probability at least $1-\delta$ (Lemma~\ref{lem: good events}), we conclude the following.
\begin{corollary}
\label{cor: main alg reliable}
    Algorithm~\ref{alg: main} is $(\epsilon, \delta)$-reliable.
\end{corollary}

The proof details of Theorem~\ref{thm: correctness} are given in Appendix~\ref{sec: appendix correctness}, and we provide a sketch here.
    Combining the guarantee from Line~\ref{line: return}, inequalities~\eqref{eq: quantile anytime bound}, and the choice of $\tilde{\epsilon} \le \lambda \cdot \epsilon/((c+1) \lambda) = \epsilon/(c+1)$ yields
    \begin{equation}
    Q_{\hat{k}}(q) >
    \mathrm{LCB}_t(\hat{k})  \ge
    \max\limits_{a \in \A_t \setminus \{ \hat{k} \} }  
    \mathrm{UCB}_t(a) -  (c+1)\tilde{\epsilon}    
    \ge 
    \max\limits_{a \in \A_t \setminus \{ \hat{k} \} }
    Q_a(q) -  \epsilon.  
    \end{equation}
    It remains to show that the optimal arm $k^*$ lies in $\A_t$ for all $t$, which we defer to Appendix~\ref{sec: appendix correctness}.

\subsection{Upper Bound}
\label{sec: upper bound}
In this section, we bound the number of arm pulls for a given instance $\nu \in \cE$. To characterize the number of arm pulls, we define the gap of each arm as follows.
\begin{definition}[Arm gaps]
\label{def: our gap}
     Fix an instance $\nu \in \cE$.
     Let $\tilde{\epsilon} = \tilde{\epsilon}(\lambda, \epsilon, c)$ and $\A_{\epsilon} = \A_{\epsilon}(\nu)$ be as defined in Algorithm~\ref{alg: main} and~\eqref{def: performance def} respectively.
    For each arm $k \in \A$, we define the gap  $\Delta_{k} =
    \Delta_{k}(\nu, \lambda, \epsilon, c, q)$ as follows:
\begin{equation}
    \label{eq: our gap}
    \Delta_{k}
    \coloneqq
    \begin{cases}
    \sup
    \Big\{
        \Delta \in \left[0, \min(q, 1-q) \right]
        :
        Q_k(q + \Delta) 
        \le
        \max\limits_{a \in \A  }
        Q^+_{a}(q - \Delta) - \tilde{\epsilon}
        \Big\}
    & \text{if }  k \not\in \A_{\epsilon} \\
   \max\limits_{\A_{\epsilon} \subseteq S \subseteq \A}
        \Delta_{k}^{(S)}
    & \text{if } k \in \A_{\epsilon} 
    \end{cases},
\end{equation}
where $Q^+_k(p)$ is the upper quantile function defined in Lemma~\ref{lem: quantile anytime bound},  
and
\begin{equation}
\label{eq: Delta k^S}
    \Delta_{k}^{(S)} \coloneqq
   \sup
    \Big\{
        \Delta \in 
       \Big[0, \min_{a \not\in S} \Delta_{a}  \Big]
        :
        Q^+_k(q - \Delta) 
        \ge 
        \max\limits_{ a \in S \setminus \{k\}} 
        Q_{a}(q + \Delta) - c \tilde{\epsilon}
        \Big\}
\end{equation}
for each subset $S$ satisfying $\A_{\epsilon} \subseteq S \subseteq \A$. We use the convention that the minimum (resp. maximum) of an empty set is $\infty$ (resp. $- \infty$).
\end{definition}

\begin{remark}[Intuition on arm gaps]
    We provide some intuition for the gap definitions:
\begin{itemize}[topsep=0pt,itemsep=0pt]
    \item 
    For an arm $k \not\in \A_{\epsilon}$, the gap
    $\Delta_{k}$ captures how much worse $k$ is than some other arm $a$. 
    When $k$ is sufficiently pulled relative to $1/\Delta_{k}$, 
    we can establish that
    $\mathrm{UCB}_t(k) \le \mathrm{LCB}_t(a)$, 
    which implies that $k$ is suboptimal, and we can stop pulling it. The details and derivation are given in Lemma~\ref{lem: elim suboptimal} and its proof.
    
    \item 

     To understand the gap $\Delta_{k} = \max\limits_{\A_{\epsilon} \subseteq S \subseteq \A}
        \Delta_{k}^{(S)}$ for a satisfying arm $k \in \A_{\epsilon}$,
    we first consider $\Delta_{k}^{(S)}$ for a fixed subset $S \supseteq \A_{\epsilon}$.
     This captures how much better arm $k$ is than the ``best'' arm $a \in S$ (up to $\epsilon$). 
     When arm $k$ is sufficiently pulled relative to $1/\Delta_{k}^{(S)}$, we can establish that 
    the termination condition is satisfied.
    Since $S \supseteq \A_{\epsilon}$ is arbitrary, we define $\Delta_{k}$ based on the set $S$ giving the highest $\Delta_{k}^{(S)}$.
     The details and derivation are given in Lemma~\ref{lem: termination} and its proof.
    
    \item When some arm $k^*$ is the only satisfying arm (i.e., $\A_{\epsilon} = \{k^*\}$), we have 
    \begin{equation}
    \label{eq: lower bound k* arm gap}
    \Delta_{k^*}  
     \ge  \Delta_{k^*}^{(\A_{\epsilon})}
     = \sup  \Big\{   \Delta \in 
       \Big[0, \min_{a \not\in \A_{\epsilon}} \Delta_a \Big] :
        Q^+_k(q - \Delta)  \ge  -\infty
        \Big\} 
    = \min\limits_{a \not\in \A_{\epsilon}} \Delta_a
    = \min\limits_{a \neq k^*} \Delta_a.
    \end{equation}
    This indicates that $k^*$ is pulled at most as many times as the smallest $\Delta_a$ value would dictate, and possibly fewer (if the while-loop terminates before $|\A_t| = 1$).    
\end{itemize}
\end{remark}

\begin{remark}[Generalization and improvement over existing arm gap]
    \label{rem: gap generalization}
    Our gap definitions were developed with the view of ensuring that we can solve essentially all solvable instances, and we will establish results of this type in Section~\ref{sec: solvable}.  Achieving this goal required several subtle choices in our gap definition, including the parameter $c$ and the optimization over $S$.
    We generalize existing gaps for the QMAB problem in the sense that those are recovered by considering $c \to \infty$, $S = \A$, and using only lower quantile functions. In Appendix~\ref{sec: appendix gap definition generalization}, we provide more details about these choices and give an instance where the gap is positive under our definition but is zero using existing definitions.
\end{remark}

\begin{remark}[Further improvement]
    \label{rem: further improvement}
     Due to the assumption that the $q$-quantile of each arm is in $[0, \lambda]$, we can improve our gap definition by replacing the terms $Q^+_{(\cdot)}(q - \Delta)$ and  $Q_{(\cdot)}(q + \Delta)$
    with $\max\big\{0, Q^+_{(\cdot)}(q - \Delta)\big\}$
    and $\min\big\{\lambda, Q_{(\cdot)}(q + \Delta)\big\}$ respectively.  
    We adopt Definition~\ref{def: our gap} to avoid further complicating the gap definition and subsequent analysis, but we will provide detailed discussion of this modified gap in Appendix~\ref{sec: appendix potential improvement}.
\end{remark}

Having defined the arm gaps, we now state an upper bound on the total number of arm pulls by Algorithm~\ref{alg: main}.

\begin{theorem}[Upper bound]
\label{theorem: upper bound}
   Fix an instance $\nu \in \cE$, and suppose Algorithm~\ref{alg: main} is run with input $(\A, \lambda, \epsilon, q, \delta)$ and parameter $c \ge 1$.
    Let $\A_{\epsilon}(\nu) $ be as defined in~\eqref{def: performance def} and let the gap $\Delta_{k} = \Delta_{k}(\nu, \lambda, \epsilon, c, q)$ be as defined in Definition~\ref{def: our gap} for each arm $k \in \A$.
    Under Event~$E$ as defined in Lemma~\ref{lem: good events},
    the total number of arm pulls is upper bounded~by    
    \begin{equation}
    \label{eq: upper bound}
        O
        \left(
        \left(
        \sum_{ k \in \A }
        \dfrac{1}{ \max\big( \Delta_{k},  \Delta  \big)^2} \cdot 
        \left( 
         \log \left(\frac{1}{ \delta } \right) +
         \log \left(\frac{1}{ \max\big( \Delta_{k},  \Delta  \big)}\right) +
         \log \left(\frac{c \lambda K}{ \epsilon } \right)    
        \right)
        \right)
        \right),
    \end{equation}
    where $\Delta  =  \Delta(\nu, \lambda, \epsilon, c, q) = \max\limits_{a \in \A_{\epsilon}(\nu)} \Delta_{a}$.
\end{theorem}

Combining Lemma~\ref{lem: good events}, Theorem~\ref{thm: correctness}, and Theorem~\ref{theorem: upper bound}, we obtain
the high-probability correctness for instances with positive gap.
\begin{corollary}
\label{cor: combined guarantee}
    Fix an instance $\nu \in \cE$ and suppose
    $\Delta =  \Delta(\nu, \lambda, \epsilon, c, q) $ as defined in Theorem~\ref{theorem: upper bound} is positive. Then, with probability at least $1-\delta$, Algorithm~\ref{alg: main} returns an arm~$\hat{k}$ satisfying~\eqref{def: performance def} and uses a total number of arm pulls satisfying~\eqref{eq: upper bound}.
\end{corollary}

We will provide near-matching lower bounds in the next section, and an impossibility result for the instances with zero gap in Section~\ref{sec: solvable}. The proof details of Theorem~\ref{theorem: upper bound} are given in Appendix~\ref{sec: appendix upper bound}, and we provide a sketch here.

\begin{proof}[Proof outline for Theorem~\ref{theorem: upper bound}]
    Under Event $E$, the while-loop of Algorithm~\ref{alg: main} terminates
    when the round index~$t$ is large enough to satisfy
    $\Delta^{(t)} \le \frac{1}{2} \Delta $,
    which happens when $t = \log_2 (1/\Delta ) + \Theta(1)$ since $\Delta^{(t)} = 2^{-t+1} \dot \min(q,1-q)$.
    Summing through the number of arm pulls $\widetilde{O}\big( \left(\Delta^{(t)}\right)^{-2}\big)$ given in~\eqref{eq: QuantEst arm pulls} for $t = 1, \ldots, \log_2 (1/\Delta ) + \Theta(1) $
    yields the upper bound $\widetilde{O}\left( \Delta ^{-2}\right)$ for each arm $k \in \A$.
    However, it is also possible that some arms are eliminated before the while-loop terminates.
    Specifically, each non-satisfying arm $k \not\in \A_{\epsilon}(\nu)$ is eliminated when the index~$t$ satisfies $\Delta^{(t)} \le \frac{1}{2} \Delta_{k}$,  which yields the upper bound $\widetilde{O}\left(\Delta_{k}^{-2}\right)$. Taking the minimum between these two gives~\eqref{eq: upper bound}.
\end{proof}

\section{Lower Bounds}
\label{sec: lower bound}

In this section, we provide two lower bounds on the number of arm pulls.  In Section~\ref{sec: lower bound unquantized}, we provide a near-matching worst-case lower bound 
$\Omega \big( \sum_{k\in\A} \Delta_k^{-2} \log(\delta^{-1})  \big)$
for instances with positive gap and $\epsilon$ is small enough such that $\A_{\epsilon}(\nu)= \{k^*\}$.
This lower bounds holds even in the absence of communication constraints.
 In Section~\ref{sec: log lambda epsilon dependence}, we address the $\log(\lambda/\epsilon)$ dependence in the upper bound by showing that $\Omega(\log(\lambda/\epsilon))$ arm pulls are needed for any $(\epsilon, \delta)$-reliable algorithm when 1-bit threshold queries are used; in particular, targeting $\epsilon = 0$ is infeasible without further assumptions.

\subsection{Lower Bound for the Unquantized Variant}
\label{sec: lower bound unquantized}
We present a worst-case lower bound on the expected number of arm pulls for the setup with no communication constraint.
The lower bound is based on a bad instance adapted from \cite[Theorem 4]{nikolakakis2021quantile}, which is for the quantile bandit problem of identifying the unique optimal arm $k^*$.
Specifically, for instances with satisfying arm set $\A_{\epsilon}(\nu) = \{k^*\}$,
the only correct output in both problem formulations is $k^*$.
By choosing $\epsilon$ to be sufficiently small, the hard instance in their problem formulation (which does not allow an $\epsilon$ relaxation) can be adapted to be a hard instance in our problem formulation.

\begin{theorem}[Worst-case lower bound]
\label{thm: lower bound unquantized}
Fix  $q, \delta \in (0, 1)$ and $\lambda \ge 1$.
There exists a quantile bandit instance~$\nu \in \cE$ with a unique best arm $k^*$ such that for any $\epsilon > 0$ satisfying
\begin{equation}
\label{eq: epsilon condition}
    \epsilon  \le
    \frac{1}{2} \big(Q_{k^*}(q) - \max \limits_{k \ne k^*} Q_k(q) \big)
\end{equation}
and any $(\epsilon, \delta)$-reliable algorithm,
 the number of arm pulls $\tau$ satisfies
\begin{equation}
\label{eq: lower bound unquantized}
\E[\tau]
\ge
\Omega \bigg(
\sum_{k=1}^K
 \frac{1}{\Delta_k^2}
    \log \left( \frac{1}{\delta} \right)
    \bigg),
\end{equation}
where $\Delta_{k}(\nu, \epsilon, q) \coloneqq \lim\limits_{c \to \infty} \Delta_{k}(\nu, \lambda, \epsilon, c, q)$ is the gap defined in Definition~\ref{def: our gap} with $c \to \infty$ (see~\eqref{eq: gap k infinite c} for the explicit form).
\end{theorem}

\begin{proof}
    See Appendix~\ref{sec: appendix unquantized lower bound}.
\end{proof}

The only difference in the upper bound~\eqref{eq: upper bound} and lower bound is that the lower bound only contains the log factor $\log(\delta^{-1})$ rather than the sum of three log factors, and so our upper bound matches the dependence on $\Delta_k$ of the lower bound to within a logarithmic factor.  We note that if $\delta \le \max(\Delta_k,\Delta)^{\Theta(1)}$ and $\delta \le \big( \frac{\epsilon}{c\lambda K} \big)^{\Theta(1)}$ then the sum of three log terms in \eqref{eq: upper bound} simplifies to $O\big( \log(\delta^{-1}) \big)$, so in this ``low error probability'' regime we in fact get matching scaling laws in the upper and lower bound.

\subsection{$\Omega(\log(\lambda/\epsilon))$ Dependence Under Threshold Query Model}
\label{sec: log lambda epsilon dependence}
In this section, we show that $\Omega(\log(\lambda/\epsilon))$ arm pulls is needed for any $(\epsilon, \delta)$-reliable algorithm in the case that only threshold queries are allowed.
That is, the side information sent by the learner to the agent is always some threshold query of the form ``Is $r_{a_t, t} \le \gamma_t$?'', and the learner receives the 1-bit comparison feedback $\mathbf{1}(r_{a_t, t} \le \gamma_t)$.
This is a common 1-bit quantization method in practice and is also the one used in Algorithm~\ref{alg: main}, though it would also be of interest to determine whether using other 1-bit quantization methods can help.

\begin{theorem}[$\Omega(\log(\lambda/\epsilon))$ dependence]
\label{thm: log lambda/epsilon dependence}
Fix $\lambda \ge \epsilon > 0$, and $q \in (0, 1)$, and $\delta \in (0, 0.5)$.
Under the threshold query model, there exists a two-arm quantile bandit instance~$\nu$ with deterministic rewards such that any $(\epsilon, \delta)$-reliable algorithm requires $\Omega(\log(\lambda/ \epsilon))$ arm pulls.
\end{theorem}

\begin{proof}
    The idea is that if the two deterministic arms in $[0,\lambda]$ are separated by $2\epsilon$, then a binary search over $\Theta(\lambda/\epsilon)$ possible choices is needed just to locate them.  See Appendix~\ref{sec: appendix log lambda epsilon dependence} for the details.
\end{proof}

 While our upper and lower bounds match to within at most logarithmic factors under mild conditions, we leave it open as to 
(i)  whether the dependence on $\Delta_k$ can be improved in general (e.g., to doubly-logarithmic as in the ones in unquantized quantile BAI~\cite{nikolakakis2021quantile, howard2022sequential}, and (ii) whether there exist regimes in which the \emph{joint} dependence on the gaps and $(\lambda,\epsilon)$ can be improved.

\section{Solvable Instances} 
\label{sec: solvable}
In Sections~\ref{sec: upper bound} and~\ref{sec: lower bound unquantized}, we provided nearly matching upper and lower bounds for instances with positive gap.
In this section, we study the ``(un)solvabality'' of bandit instances with zero gap, and show that essentially all bandit instances that are ``solvable'' have positive gap, as long as parameter~$c$ is large enough (see Remark~\ref{rem: picking large enough c}).
To formalize this idea, we define the following class of bandit instances.

\begin{definition}[Solvable instances]
\label{def:solvable}
     Let $\A, \epsilon,$ and $q$ be fixed. 
     We say that an instance $\nu \in \cE$ is $\epsilon$-\emph{solvable} if for each $\delta \in (0, 1)$, there exists an algorithm that is $(\epsilon, \delta)$-reliable and it holds under instance $\nu$ that\footnote{We could require that $\PP_{\nu}[\tau < \infty] = 1$ in this case and the subsequent analysis and conclusions would be essentially unchanged.  Recall also that $\PP_{\nu}[\cdot]$ denotes probability under instance $\nu$.}
    \begin{equation}
        \PP_{\nu}[\tau < \infty \cap \hat{k}\in\Ac_{\epsilon}]\ge1-\delta.
    \end{equation}
    If no such algorithm exists, we say that $\nu$ is $\epsilon$-\emph{unsolvable}.
\end{definition}

\begin{remark}
\label{rem: solvable inclusion}
    Fix $0 < \epsilon_1 \le \epsilon_2$. If an instance $\nu$ is
    $\epsilon_1$-solvable, then it is $\epsilon_2$-solvable.
    This follows directly from $\A_{\epsilon_1}(\nu) \subseteq \A_{\epsilon_2}(\nu)$.
\end{remark}

From Corollary~\ref{cor: combined guarantee}, we deduce that any instance with a positive gap is solvable. 
 
\begin{corollary}[Positive gap is solvable]
\label{cor: positive gap is solvable}
    Let $\A, \lambda, \epsilon, q,$ and $c$ be fixed. 
    Suppose an instance $\nu$ satisfies
    $ \Delta > 0$, where $\Delta =  \Delta(\nu, \lambda, \epsilon, c, q) $ is as defined in Theorem~\ref{theorem: upper bound}.
    Then $\nu$ is $\epsilon$-solvable. 
\end{corollary}
The main result of this section is that the reverse inclusion nearly holds, in the following sense.
 \begin{theorem}[Zero gap is unsolvable]
 \label{thm: zero gap is unsolvable}
    Let $\lambda, \epsilon, c,$ and $q$ be fixed, 
    and let $\tilde{\epsilon} = \tilde{\epsilon}(\lambda, \epsilon, c)$ be as defined in 
    Algorithm~\ref{alg: main}.
    Suppose an instance $\nu \in \cE$ satisfies $\Delta(\nu, \lambda, \epsilon, c, q) = 0$.
    If we assume for $\nu$ that there exists some sufficiently small $\eta_0 > 0$ such that 
    $0 \le Q_k^+(q-\eta_0) \le Q_k(q+\eta_0) \le \lambda$, then $\nu$ is $c\tilde{\epsilon}$-unsolvable.
 \end{theorem}
\begin{proof}
    See Appendix~\ref{sec: appendix solvable instance}.
\end{proof}

\begin{remark}[Removing the additional assumption]
\label{rem: remove additional assumption}
    The additional assumption involving $\eta_0$ 
    is mild; it is trivially satisfied by instances with all reward distributions supported on $[0, \lambda]$, and also holds significantly more generally since $\eta_0$ can be arbitrarily small.
    Moreover, in Appendix~\ref{sec: assumption removal}, we show that
    this assumption is unnecessary if we use the modified gap (see Remark~\ref{rem: further improvement}) instead of $\Delta$.
\end{remark}

\begin{remark}
\label{rem: picking large enough c}
 For each $\theta \in (0, 1)$, picking $c = \lceil 2\theta / (1-\theta)\rceil$ yields
\begin{equation}
       \nu \text{ is } \theta\epsilon\text{-solvable} 
    \implies
    \nu \text{ is } c\tilde{\epsilon}\text{-solvable}
    \implies
     \Delta(\nu, \lambda, \epsilon, c, q) > 0
     \implies 
     \nu \text{ is } \epsilon\text{-solvable},
\end{equation}
where the last two implications follow from Theorem~\ref{thm: zero gap is unsolvable} and Corollary~\ref{cor: positive gap is solvable}, and the first implication follows from Remark~\ref{rem: solvable inclusion} and the following inequality:
\begin{equation}
    c \tilde{\epsilon} 
    = 
    \frac{c \lambda}{ \left\lceil (c+1) \lambda/\epsilon \right\rceil}
    \ge
     \frac{c \lambda}{  (c+2) \lambda / \epsilon }
     = 
     \left(1 - \frac{2}{c+2} \right) \epsilon
     \ge
     \left(1 - \frac{2}{\frac{2\theta+2-2\theta}{1-\theta}} \right) \epsilon
     = \theta \epsilon.
\end{equation} 
Since $\theta$ can be arbitrarily close to $1$,
we have $\Delta(\nu, \lambda, \epsilon, c, q) > 0$
     for essentially all $\epsilon$-solvable instances by picking a sufficiently large $c$.
 \end{remark}

The proof of Theorem~\ref{thm: zero gap is unsolvable} will turn out to directly extend to a ``limiting'' version in which we replace $c\tilde{\epsilon}$ by $\lim\limits_{c \to \infty} c\tilde{\epsilon} = \epsilon$ and $\Delta(\nu, \lambda, \epsilon, c, q)$ by $\lim\limits_{c \to \infty} \Delta(\nu, \lambda, \epsilon, c, q)$, giving the following corollary.

 \begin{corollary}
 \label{cor: zero gap is unsolvable}
    Let $\lambda, \epsilon$, and $q$ be fixed.
    Let $\Delta_{k}(\nu, \epsilon, q)$ be the gap defined in Definition~\ref{def: our gap} with $c \to \infty$ (see~\eqref{eq: gap k infinite c} for the explicit form).    
    Suppose an instance $\nu \in \cE$ satisfies $\Delta(\nu, \epsilon, q) = \max\limits_{k \in \A_{\epsilon(\nu)}} \Delta_k(\nu, \epsilon, q) = 0$.
    If we assume for $\nu$ that there exists some sufficiently small $\eta_0 > 0$ such that 
    $0 \le Q_k^+(q-\eta_0) \le Q_k(q+\eta_0) \le \lambda$, then $\nu$ is $\epsilon$-unsolvable.
 \end{corollary}
 \begin{proof}
    See Appendix~\ref{sec: appendix solvable instance}.
\end{proof}

\acks{This work was supported by the Singapore National Research Foundation under its AI Visiting Professorship programme.}

\newpage
\bibliography{bibliography}

\newpage
\appendix

\section{Quantile Estimation Subroutine}
\label{sec: appendix QuantEst}

\subsection{Noisy Binary Search}
\label{sec: appendix MNBS reformulation}
We first momentarily depart from MAB and discuss the monotonic noisy binary search (MNBS) problem of~\cite{karp2007noisy}; see also the end of Appendix~\ref{sec: appendix quant est related work} for a summary of some related work on noisy binary search.
The original problem formulation was stated in terms of finding a special coin $i$ among $n$ coins, but this can be restated as follows: 
We have a random variable~$R$ with an unknown CDF $F$ and a list of $n$ points $x_1  \le \cdots \le x_n$ such that $Q(\tau) \in [x_1, x_n]$, and the goal is to find an index $i$  satisfying 
\begin{equation}
\label{eq: MNBS guarantee}
    [F(x_i), F(x_{i+1})] \cap  (\tau - \Delta, \tau + \Delta) \ne 
    \varnothing
\end{equation}
via adaptive queries of the form $\mathbf{1}(R \le x_j)$. Note that each query $\mathbf{1}(R \le x_j)$ is an independent Bernoulli random variable with parameter $F(x_j)$.
We will make use of the following main result from~\cite[Theorem 1.1]{gretta2023sharp}.

\begin{proposition}[Noisy binary search guarantee]
\label{prop: MNBS guarantee}
For any $\delta \in (0, 1)$ and relaxation parameter $\Delta \le \min(\tau, 1-\tau)$, the MNBS algorithm in \cite{gretta2023sharp}
output an index $i$ after at most
$O \big( \frac{1}{\Delta^2}  \log  \frac{n}{\delta} \big)$
queries\footnote{\label{footnote: constants MNBS}The expression for the number of iterations in \cite{gretta2023sharp} is more complicated because it has some terms with explicit constant factors, but in $O(\cdot)$ notation it simplifies to $O \big( \frac{1}{\Delta^2}  \log  \frac{n}{\delta} \big)$. We do not specify the exact number of loops in Algorithm~\ref{alg: quantile interval}, as doing so is somewhat cumbersome and the focus of our work is on the scaling laws.} and $i$ satisfies~\eqref{eq: MNBS guarantee} with probability at least $1- \delta$. 
\end{proposition}
The bulk of the MNBS algorithm in \cite{gretta2023sharp} is based on Bayesian multiplicative weight updates: Start with a uniform prior over which of the $n$ intervals crosses quantile $\tau$, make the query at $x_j$ whose $F(x_j)$ is nearest to $\tau$ under current distribution, update the posterior by multiplying intervals on one side of the query by $1 + c \Delta$ and the other side by $1 - c \Delta$ for some fixed constant $c$, and repeat. Other MNBS algorithms such as those in \cite{karp2007noisy}, or even a naive binary search with repetitions (see \cite[\S 1.2]{karp2007noisy}), could also be used to solve the MNBS problem, but we choose  \cite{gretta2023sharp} since it has the best known scaling of the query complexity. Further comparisons of the relevant theoretical guarantees and practical performance can be found in \cite{gretta2023sharp}.

\subsection{Quantile Estimation with 1-bit Feedback}
\label{sec: appendix quant est related work}
The MNBS algorithm can be implemented under our 1-bit communication-constrained setup. Specifically, the learner decides which arm $k$ to query as well as the point $x_j$  to query, and then sends a threshold
query ``Is $R_k \le x_j$?'' as side information to the agent, where $R_k$ is the random reward (variable) of the arm~$k$ with CDF $F_k$. The agent will then pull arm $k$ and reply with a 1-bit binary feedback corresponding to the observation. Note that 
while the $O \big( \frac{1}{\Delta^2}  \log  \frac{n}{\delta} \big)$ queries for a given arm are done in an adaptive manner, the queries themselves can be requested at different time steps without any requirement of agent memory. 
A high-level description of the implementation for a fixed arm is given in Algorithm~\ref{alg: quantile interval}. This gives us the following guarantee, which is a simple consequence of Proposition~\ref{prop: MNBS guarantee}.

\begin{algorithm}
    \caption{Communication-constrained quantile estimation subroutine ($\mathtt{QuantEst}$ in
Algorithm~\ref{alg: main})}
    \label{alg: quantile interval}

   \textbf{Input}: 
    Arm with reward $R$ distributed according to CDF $F$,
    a list $\mathbf{x}$ of $n$ points $x_1 \le  \cdots \le x_n$,
    quantile $\tau \in (0, 1)$ satisfying $Q(\tau) \in [x_1, x_n]$,
    approximation parameter~$\Delta \le \min(\tau, 1-\tau)$,
    error probability $\delta \in (0,1)$
    
    \textbf{Output} Index $i \in \{1, \dots, n-1\}$    

\begin{algorithmic}[1]
      
        \For{$t = 1$ to $t_{\max}$ (with\footref{footnote: constants MNBS} $t_{\max} = O \big( \frac{1}{\Delta^2}  \log  \frac{n}{\delta} \big)$)}

          \State \textbf{At Learner:}
          
          \State~~~~Pick index $j$
          according to Bayesian weight update as in~\cite{gretta2023sharp}
          
          \State~~~~Send threshold query 
          ``Is $R \le x_j$?'' to the agent

            \State \textbf{At Agent:}
          
          \State~~~~Pull arm and observe reward $r$
          
          \State~~~~Send 1-bit feedback $\mathbf{1}(r \le x_j)$ to the learner

        \EndFor

    \State Return index $i$  according to~\cite{gretta2023sharp}

\end{algorithmic}    
\end{algorithm}

\begin{corollary}[$\mathtt{QuantEst}$ guarantee]
\label{cor: QuantEst guarantee}
Let $(F, \mathbf{x}, \tau, \Delta, \delta)$ be a valid input of Algorithm~\ref{alg: quantile interval}, and let~$n$ be the number of points in $\mathbf{x}$. 
Then the algorithm outputs  an index $i$ after at most
$O \big( \frac{1}{\Delta^2}  \log  \frac{n}{\delta} \big)$
queries and $i$  satisfies
    $\mathbb{P}
    \left([F(x_i), F(x_{i+1})] \cap  (\tau - \Delta, \tau + \Delta) = \varnothing \right) < \delta$.
\end{corollary}

\textbf{Related work on noisy binary search and quantile estimation.}
We briefly recap the original MNBS problem in \cite{karp2007noisy, gretta2023sharp}:
There are $n$ coins whose unknown probabilities $p_j \in [0, 1]$ are sorted in nondecreasing order, where flipping coin $j$ results in head with probability $p_j$. The goal is to identify a coin $i$ such that the interval $[p_i, p_{i+1}]$ has a nonempty intersection with $(\tau - \Delta, \tau + \Delta)$. This model subsumes noisy binary search with a fixed noise level \cite{burnashev1974interval, ben2008bayesian, dereniowski2021noisy, gu2023optimal} (where $p_j = \frac{1}{2} - \Delta$ for $j \leq i$ and $p_j = \frac{1}{2} + \Delta$ otherwise) as well as regular binary search (where $p_j \in \{0, 1\}$). As we discussed in Appendix~\ref{sec: appendix MNBS reformulation}, this problem can be reformulated into the problem of estimating (the quantile of) a distribution using threshold/comparison queries, where the noise in the feedback is stochastic.  This quantile estimation problem has been generalized to a non-stochastic noise setting \cite{meister2021learning, okoroafor2023non}, and was also studied in the context of online dynamic pricing and auctions \cite{kleinberg2003value, leme2023pricing, leme2023description}. 
In particular, \cite[Algorithm~1]{leme2023pricing} is similar to Algorithm~\ref{alg: quantile interval} (or equivalently subroutine $\mathtt{QuantEst}$ used on Lines~\ref{ltk def} and~\ref{utk def} of Algorithm~\ref{alg: main}), in the sense that both use noisy binary search to identify the quantile of a \textit{single} distribution. However, they use the naive binary search with repetitions to form confidence intervals containing the quantile, which has a suboptimal complexity $O \big( \frac{1}{\Delta^2} \log n \log  \frac{\log n}{\delta} \big)$; see \cite[\S 1.2]{karp2007noisy} for details.  Overall, while ideas from the existing literature on quantile estimation of a \textit{single} distribution with threshold queries may provide useful context, they do not readily translate into Algorithm~\ref{alg: main} or the analysis that led to our main contributions.

\subsection{Proof of Lemma~\ref{lem: good events} (Bounding the Probability of Event E) }
\label{sec: proof event E}
\begin{proof}[Proof of Theorem~\ref{lem: good events}]
    For a fixed $t \ge 1$ and a fixed $k \in \mathcal{A}_t$,
    we have
    \begin{equation}
        \Pr{
        \overline{E_{t, k, l}}
        }
        \le  \frac{\delta \cdot \Delta^{(t)}}{2 |\mathcal{A}_t|}
        \quad
        \text{and}
        \quad
         \Pr{
        \overline{E_{t, k, u}}
        }
        \le  \frac{\delta \cdot \Delta^{(t)} }{2 |\mathcal{A}_t|}
        \quad
    \end{equation}
    by the guarantee of the $\mathtt{QuantEst}$ (see Corollary~\ref{cor: QuantEst guarantee}).
    Applying the union bound, we obtain
    \begin{equation}
        \Pr{\overline{E}}
        \le 
        \sum_{t \ge 1}
        \sum_{k \in \mathcal{A}_t} 
        \frac{\delta \cdot \Delta^{(t)}}{ |\mathcal{A}_t|}
        \le 
        \sum_{t \ge 1}
        |\mathcal{A}_t| \cdot 
        \frac{\delta \cdot \Delta^{(t)}}{ |\mathcal{A}_t|}
        =
        \delta
         \sum_{t \ge 1}
         \Delta^{(t)}
         =
         \delta
         \sum_{t \ge 1}
         2^{-t} 
         \le \delta
    \end{equation}
    as desired. The number of arm pulls~\eqref{eq: QuantEst arm pulls} follows immediately from the guarantee of $\mathtt{QuantEst}$ from Corollary \ref{cor: QuantEst guarantee}, $|\A_t| \le |\A| = K$, 
    and the number of points $n = \Theta(c\lambda/\epsilon)$.
\end{proof}
\section{Proof of Lemma~\ref{lem: quantile anytime bound} (Anytime Quantile Bounds)}
\label{sec: appendix anytime quantile bounds}
We first present a useful auxiliary lemma.
\begin{lemma}
    Under the setup of Lemma~\ref{lem:  quantile anytime bound} (including Event $E$ from Lemma~\ref{lem: good events} holding), we have the following bounds:
    \begin{equation}
    \label{eq: xltk upper bound}
        x_{l_{t, k}} < Q_k(q)
    \end{equation}
    \begin{equation}
    \label{eq: xltk+1 lower bound}
        Q^+_k\big( q -  \Delta^{(t)} \big)
        \le x_{l_{t, k} + 1}
    \end{equation}
    \begin{equation}
    \label{eq: xutk upper bound}
        x_{u_{t, k}} < Q_k\big( q +  \Delta^{(t)} \big)
    \end{equation}
    \begin{equation}
    \label{eq: xutk+1 lower bound}
        Q^+_k(q)
        \le x_{u_{t, k} + 1}
    \end{equation}
    for each round $t \ge  1$ and arm $k \in \A_t$.
\end{lemma}
\begin{proof}
We will prove only~\eqref{eq: xltk upper bound} and~\eqref{eq: xltk+1 lower bound}
for an arbitrary $t  \ge 1$ and  $k \in \A_t$
in detail, as~\eqref{eq: xutk upper bound} and~\eqref{eq: xutk+1 lower bound} can be proved similarly.
Observe that, under event $E_{t,k,l} \subset E$ (see \eqref{eq: event Etkl}), we have 
\begin{equation}
\label{eq: Fk_xltk bound}
    F_k(x_{l_{t, k}}) < q
    \quad \text{and} \quad
    q -  \Delta^{(t)} < F_k(x_{l_{t, k}+1})
\end{equation}
respectively,
as otherwise the interval $[F_k(x_{l_{t, k}}), F_k(x_{l_{t, k}+1})]$
would fall on the right and the left, respectively, of the interval $\left( q - \Delta^{(t)}, q   \right)$. A similar argument through the event $E_{t,k,u} \subset E$ (see \eqref{eq: event Etku}) yields
\begin{equation}
\label{eq: Fk_xutk bound}
    F_k(x_{u_{t, k}}) < q + \Delta^{(t)}
    \quad \text{and} \quad
    q  < F_k(x_{u_{t, k}+1}).
\end{equation}

We now prove~\eqref{eq: xltk upper bound} using~\eqref{eq: Fk_xltk bound}; the inequality~\eqref{eq: xutk upper bound} can be proved similarly through~\eqref{eq: Fk_xutk bound}. If $x_{l_{t, k}} = - \infty$, then~\eqref{eq: xltk upper bound} holds trivially.
Therefore, we proceed on the assumption that $x_{l_{t, k}} \in \R$.
Then, using standard properties of quantile functions (see, e.g., \cite[4.3 Theorem]{dufour1995distribution}), we have $x_{l_{t, k}} < Q_k(q)$ as desired.

We now prove~\eqref{eq: xltk+1 lower bound} using~\eqref{eq: Fk_xltk bound}; the inequality~\eqref{eq: xutk+1 lower bound} can be proved similarly through~\eqref{eq: Fk_xutk bound}. 
If $x_{l_{t, k}+1} = \infty$, then~\eqref{eq: xltk+1 lower bound} holds trivially.
Therefore, we proceed on the assumption that $x_{l_{t, k}+1} \in \R$.
In this case, it is a finite upper bound on the values in the set
    $\{ z \in \R: F_k(z) \le q -  \Delta^{(t)}  \}$, and so this set has a finite supremum. It follows that 
    \begin{equation}
        x_{l_{t, k}+1} \ge 
        \sup \{ z \in \R : F_k(z) \le q -  \Delta^{(t)}  \} =
        Q^+_k\big( q -  \Delta^{(t)} \big)
    \end{equation}
    as desired.
\end{proof}

\begin{proof}[Proof of Lemma~\ref{lem:  quantile anytime bound}]
We break down the bounds into  inequalities as follows:
\begin{multicols}{2}
\begin{enumerate}[label=(\roman*)]

    \item  $\mathrm{LCB}_{\tau}(k) \le \mathrm{LCB}_t(k)$
    
    \item $\mathrm{LCB}_t(k)  < Q_k(q)$
    
    
    \item $Q_k(q) \le \mathrm{UCB}_t(k)$
    
    \item $\mathrm{UCB}_t(k)  \le \mathrm{UCB}_{\tau}(k)$
    
    \item $Q^+_k\big(q -  \Delta^{(t)} \big)  \le \mathrm{LCB}_t(k) + \tilde{\epsilon}$

    \item $\mathrm{UCB}_t(k) - \tilde{\epsilon}
        <
         Q_k\big(q + \Delta^{(t)} \big)$
\end{enumerate}
\end{multicols}
We will prove only inequalities (i), (ii), and (iv)
for an arbitrary $t  > \tau \ge 0$ and  $k \in \A_t$
in detail, as all the other inequalities can be proved similarly.

Inequality (i) follows immediately from Line~\ref{LCB definition} of Algorithm~\ref{alg: main} and induction. Likewise, we can show (iv) using Line~\ref{UCB definition} of Algorithm~\ref{alg: main}.
 
We now show inequality (ii) by induction on $t$;
inequality (iii) can be proved similarly.
    For the base case $t = 1,$
    we have
    \begin{equation}
        \mathrm{LCB}_1(k) =
        \max \left( x_{l_{t, k}}, \mathrm{LCB}_{0}(k) \right) =
        \max \left( x_{l_{t, k}}, 0 \right) =
        x_{l_{t, k}} < Q_k(q),
    \end{equation}
    where the last inequality follows from~\eqref{eq: xltk upper bound}.   
    For the inductive step, suppose that $\mathrm{LCB}_t(k) < Q_k(q)$ for a fixed $t \ge 1$. Since $x_{l_{t, k}} < Q_k(q)$, we have
    \begin{equation}
        \mathrm{LCB}_{t+1}(k) =
            \max
            \left(
            x_{l_{t, k}},
            \mathrm{LCB}_{t}(k)
            \right)
            < Q_k(q)
    \end{equation}
    as desired.

    We now show inequality (v) using~\eqref{eq: xltk+1 lower bound}; inequality (vi)
         can be shown using a similar argument through~\eqref{eq: xutk upper bound}.
        We consider three cases for the index $l_{t, k}$:
        \begin{itemize}
            \item ($l_{t, k} = 0$) In this case, we have 
                    $x_{l_{t, k} + 1} = x_{1} =  0 = \mathrm{LCB}_0(k)$, and so
                    \begin{equation}
                       Q^+_k\big( q -  \Delta^{(t)} \big)
                        \le 
                        x_{l_{t, k} + 1} =
                        \mathrm{LCB}_0(k)
                        \le 
                        \mathrm{LCB}_t(k)
                        < \mathrm{LCB}_t(k) + \tilde{\epsilon},
                    \end{equation}
                    where the first inequality follows from~\eqref{eq: xltk+1 lower bound} and the second inequality follows from inequality~(i).
            
            \item ($1 \le l_{t, k} \le n$) In this case, we have
                    \begin{equation}
                   Q^+_k\big( q -  \Delta^{(t)} \big)
                    \le x_{l_{t, k} + 1}
                    = x_{l_{t, k}} + \tilde{\epsilon}
                    \le \mathrm{LCB}_t(k) + \tilde{\epsilon},
                \end{equation}
                where the first inequality follows from~\eqref{eq: xltk+1 lower bound}, the equality follows from distance between
                consecutive points set in Line~\ref{line: list of points} of Algorithm~\ref{alg: main}, 
                and the last inequality follows from Line~\ref{LCB definition} of Algorithm~\ref{alg: main}.

           \item  ($l_{t, k} = n+1$) In this case, we have  $x_{l_{t, k}} = x_{n+1} = \lambda \ge Q_k(q) \ge  Q^+_k\big( q -  \Delta^{(t)} \big)$,
           and so
                    \begin{equation}
                       Q^+_k\big( q -  \Delta^{(t)} \big) \le 
                        x_{l_{t, k}}  \le 
                        \mathrm{LCB}_t(k)
                        < \mathrm{LCB}_t(k) + \tilde{\epsilon},
                    \end{equation}
                    where the second inequality follows from Line~\ref{LCB definition} of Algorithm~\ref{alg: main}.
        \end{itemize}
    Combining all three cases, we have
    $Q^+_k\big( q -  \Delta^{(t)} \big)
        \le \mathrm{LCB}_t(k) + \tilde{\epsilon}$ as desired.
\end{proof}

\section{Proof of Theorem~\ref{thm: correctness} (Reliability of Algorithm~\ref{alg: main})}
\label{sec: appendix correctness}
\begin{proof}[Proof of Theorem~\ref{thm: correctness}]
    We first show by induction that an optimal arm $k^*$ of instance $\nu$ (i.e., one having the highest $ q$-quantile) will always be active, i.e., $k^* \in \A_t$ for each round $t \ge 1$. 
        For the base case $t = 1$, we have $k^* \in \{1, \dots, K\} = \A_1$ trivially. We now show the inductive step: if $k^* \in \A_t$ holds, then $k^* \in \A_{t+1}$. For all arms $a \in \A_t$, we have
    \begin{equation}
        \mathrm{UCB}_t(k^*)
        \ge
        Q_{k^*}(q) 
        \ge Q_{a}(q)  
        > 
        \mathrm{LCB}_t(a),
    \end{equation}
    where the second inequality follows from the optimality of  arm $k^*$,
    while the other two inequalities follow from the anytime quantile bounds (Lemma~\ref{lem: quantile anytime bound}).
    It follows that
    $\mathrm{UCB}_t(k^*) >
            \max\limits_{a \in \mathcal{A}_{t}} \mathrm{LCB}_t(a)$,
    and so $k^* \in \A_{t+1}$ by definition (see Line~\ref{line: active arm} of Algorithm~\ref{alg: main}).

    We now argue that if Algorithm~\ref{alg: main} terminates, then the returned arm~$\hat{k}$ satisfies~\eqref{def: performance def}. 
    If Algorithm~\ref{alg: main} terminates,
    then the while-loop (Lines~\ref{line: start while loop}--\ref{line: end while loop}) must have terminated and therefore the returned arm $\hat{k}$
    satisfies the condition
    \begin{equation}
    \label{eq: k condition}
    \mathrm{LCB}_t(\hat{k})  \ge
                \max\limits_{a \in \A_t \setminus \{ \hat{k} \} }  
                \mathrm{UCB}_t(a) -  (c+1)\tilde{\epsilon}
                 \ge
          \max\limits_{a \in \A_t \setminus \{ \hat{k} \} }  
                \mathrm{UCB}_t(a)  - \epsilon,
    \end{equation}
    where the second inequality follows from Lines~\ref{line: number of points}--\ref{line: tilde epsilon} of Algorithm~\ref{alg: main}: $\tilde{\epsilon} \le \lambda \cdot \epsilon/((c+1) \lambda) = \epsilon/(c+1)$.
    If $\hat{k} = k^*$, then the returned arm satisfies~\eqref{def: performance def} trivially. Therefore, we assume that $\hat{k} \ne k^*$ for the rest of the proof. In this case, we have
    \begin{equation}
        Q_{\hat{k}}(q) >
        \mathrm{LCB}_t(\hat{k})  \ge
        \max\limits_{a \in \A_t \setminus \{ \hat{k} \} } 
        \mathrm{UCB}_t(a) -  
        \epsilon
        \ge 
        \mathrm{UCB}_t(k^*) -  \epsilon
        \ge
        \max\limits_{a \in \A_t \setminus \{ \hat{k} \} }
        Q_a(q) -  \epsilon.
    \end{equation}
    where the first and the last inequalities follow from the anytime quantile bounds (see Lemma~\ref{lem: quantile anytime bound}), while the second inequality follows from the condition~\eqref{eq: k condition} and the third inequality follows from $k^* \in \A_t$ (see above) and the assumption that $\hat{k} \ne k^*$.
\end{proof}
\section{Details on Remark~\ref{rem: gap generalization} (Comparison to Existing Gap Definitions)}
\label{sec: appendix gap definition generalization}

We first recall some existing arm gap definitions for the exact quantile bandit problem (i.e., $\epsilon = 0$) in the setting of unquantized rewards.
In \cite[Definition 2]{nikolakakis2021quantile}, the authors defined the gap $\Delta_k^{\mathrm{NKSS}} $ for each suboptimal arm $k \ne k^*$ by
\begin{equation}
\label{eq: gap NKSS}
    \Delta_k^{\mathrm{NKSS}} \coloneqq
     \sup
    \{
         \Delta \in \left[0, \min(q, 1-q) \right]
        : 
        Q_k(q + \Delta)
        \le
        Q_{k^*}(q - \Delta)
    \}.
\end{equation}
While the authors did not define the arm gap for $k^*$, we can take it to be the same as the gap of the ``best'' suboptimal arm, as their algorithm terminates only when all suboptimal arms are eliminated.
On the other hand, the arm gap defined in
\cite[(Eq. (27)]{howard2022sequential}
is given by
\begin{equation}
\label{eq: gap HR}
    \Delta_k^{\mathrm{HR}} \coloneqq
    \begin{cases}
     \sup
    \{
         \Delta \in \left[0, \min(q, 1-q) \right]
        : 
        Q_k(q + \Delta)
        \le
        \max\limits_{a \in \A}
        Q_a(q - \Delta)
    \}
    & \text{if } k \ne k^*
    \\
     \sup
    \{
        \Delta 
        \in \left[0, q \right] :
        Q_k(q - \Delta)
        \ge
        \max\limits_{a \neq k}
        Q_{a}\big(q + \Delta_a^{\mathrm{HR}}
        \big)
    \}
    & \text{if } k = k^* 
    \end{cases}.
\end{equation}
Similar to our arm gap definition (Definition~\ref{def: our gap}), the gaps $\Delta_k^{\mathrm{HR}}$ for suboptimal arms $k \ne k^*$ are not defined based on the quantile function of $k^*$. It follows that $\Delta_a^{\mathrm{HR}} \ge \Delta_a^{\mathrm{NKSS}}$ for all arms $a \in \A$.

We now study the effect of taking $c \to \infty$ in our gap, which is given below in~\eqref{eq: gap k infinite c}. From~\eqref{eq: gap k infinite c}, it is straightforward to verify that~\eqref{eq: gap HR} is recovered from our gap (Definition~\ref{def: our gap}) by using only lower quantile functions and taking $S = \A$ and $c \to \infty$. 

\textbf{Effect of parameter $c$ in the gap definition.}
For any $1 \le c_1 \le c_2$, let 
\begin{equation}
\label{eq: tilde eps 1 and 2}
    \tilde{\epsilon_1} = \frac{\lambda} {\left\lceil (c_1+1) \lambda/\epsilon \right\rceil}
    \quad 
    \text{and}
    \quad
    \tilde{\epsilon_2} =  \frac{\lambda} {\left\lceil (c_2+1) \lambda/\epsilon \right\rceil}
\end{equation} 
be as defined 
using Lines~\ref{line: number of points}--\ref{line: tilde epsilon} of 
in Algorithm~\ref{alg: main}. It can readily be verified that 
\begin{equation}
\label{eq: c1 tilde eps 1 and 2}
    \tilde{\epsilon_1} \ge \tilde{\epsilon_2}
    \quad \text{and} \quad
    c_1 \tilde{\epsilon_1} \le c_2 \tilde{\epsilon_2} \le \epsilon
    \quad \text{and} \quad
    \Delta_{k}(\nu, \lambda, \epsilon, c_1, q)
\le \Delta_{k}(\nu, \lambda, \epsilon, c_2, q).
\end{equation}
Since $\lim\limits_{c \to \infty}  \tilde{\epsilon} = 0$
and $\lim\limits_{c \to \infty} c \tilde{\epsilon} = \epsilon$,
the gap as defined in Definition~\ref{def: our gap} converges to a quantity $\Delta_{k} \coloneqq \Delta_{k}(\nu, \epsilon, 
    q) = \lim\limits_{c \to \infty}
    \Delta_{k}(\nu, \lambda, \epsilon, c, q)$, given by
\begin{equation}
\begin{aligned}
\label{eq: gap k infinite c}
   \Delta_{k} =
    \begin{cases}
    \sup
    \left\{
        \Delta \in \left[0, \min(q, 1-q) \right]
        \colon
        Q_k(q + \Delta) 
        \le
        \max\limits_{a \in \A  }
        Q^+_{a}(q - \Delta) 
        \right\}
    &\hspace{-2mm} \text{if }  k \not\in \A_{\epsilon} \\
   \max\limits_{\A_{\epsilon} \subseteq S }
   \left\{
        \sup
    \Big\{
        \Delta \in 
       \Big[0, \min\limits_{a \not\in S} \Delta_{a}  \Big]
        \colon
        Q^+_k(q - \Delta) 
        \ge
        \max\limits_{ a \in S \setminus \{k\}} 
        Q_{a}(q + \Delta) - \epsilon
        \Big\}
        \right\}
    &\hspace{-2mm} \text{if } k \in \A_{\epsilon} 
    \end{cases}.
\end{aligned}
\end{equation}
    Note that $\Delta_k$ is independent of $c$ and $\lambda$.

\begin{remark}[Use of upper quantile function]
    \label{rem: upper quantile}
    To our knowledge, we are the first to incorporate upper quantile functions in the gap definition.
    This may lead to a potentially larger arm gap as compared to defining using only lower quantile functions (e.g., changing $Q_a^+(\cdot)$ and $Q_k^+(\cdot)$ in~\eqref{eq: our gap} and~\eqref{eq: Delta k^S} to $Q_a(\cdot)$ and $Q_k(\cdot)$ respectively), and hence a tighter upper bound.
\end{remark}

\begin{remark}[Dependency on $Q_{k^*}(q-\Delta)$]
    Existing papers using an elimination-based algorithm have their arm gaps defined according to $Q_{k^*}(q-\Delta)$; see~\eqref{eq: gap NKSS} for an example.
    In contrast, we remove this dependency and define using $\max\limits_{a \in \A  }
        Q^+_{a}(q - \Delta)$, which may lead to a tighter upper bound.  
    The resulting analysis required is more challenging --
    see the discussion in Remark~\ref{rem: elim suboptimal}. 
\end{remark}

Since our gap definitions generalizes existing gap definitions, we expect that their gaps being positive on an instance $\nu$ would imply our gap being positive on $\nu$. That is, their gaps being positive is a sufficient condition for Algorithm~\ref{alg: main} to return a satisfying arm with high-probability (see Corollary~\ref{cor: combined guarantee}).

\begin{proposition}
\label{prop: generalized formulation}
     Fix an instance $\nu \in \cE$ that has a unique arm $k^*$ with the highest $q$-quantile. Let $\Delta_a^{\mathrm{NKSS}}$ and 
$\Delta_a^{\mathrm{HR}}$ be as defined in~\eqref{eq: gap NKSS} and \eqref{eq: gap HR} for each $a \in \A$.
If $\min\limits_{a \in \A} 
\left\{ \Delta_a^{\mathrm{NKSS}} \right\} > 0$
 or
 $\min\limits_{a \in \A} 
\left\{ \Delta_a^{\mathrm{HR}} \right\} > 
 0$, then $\Delta =  \Delta(\nu, \lambda, \epsilon, c, q) $ as defined in Theorem~\ref{theorem: upper bound} is also positive.
\end{proposition}

\begin{proof}
    It suffices to consider the case
    $ \min\limits_{a \in \A}  
    \left\{ \Delta_a^{\mathrm{HR}} \right\} > 0$,
    since $\Delta_a^{\mathrm{HR}} \ge \Delta_a^{\mathrm{NKSS}}$ for all arms $a \in \A$.
    Let $\eta = \frac{1}{2} \min\limits_{a \in \A}  \Delta_a^{\mathrm{HR}} > 0$.
    Then we have
    \begin{equation}
    \label{eq: positive HR implies positive gap}
        Q_{k^*}^+(q - \eta)
        \ge 
         Q_{k^*}(q - \eta)
        \ge 
        \max\limits_{a \neq k}
        Q_{a}\big(q + \Delta_a^{\mathrm{HR}}\big)
        \ge
        \max\limits_{a \in \A \setminus \{k^*\} } Q_a(q + \eta) - c \tilde{\epsilon},
    \end{equation}
    where the second inequality follows from~\eqref{eq: gap HR} 
    and $\tilde{\epsilon} = \tilde{\epsilon}(\lambda, \epsilon, c)$ is as defined in Algorithm~\ref{alg: main}.
    Combining~\eqref{eq: positive HR implies positive gap} and~\eqref{eq: Delta k^S} of our gap definition,  we have
    \begin{equation}
        \max\limits_{a \in \A_{\epsilon}} \Delta_{a} 
        \ge 
        \Delta_{k^*}  = \max\limits_{\A_{\epsilon} \subseteq S \subseteq \A}
        \Delta_{k^*}^{(S)} 
        \ge
        \Delta_{k^*}^{(\A)} \ge \eta > 0
    \end{equation}
    as desired.
\end{proof}

We now show that the converse is not true in general. 
In other words, there exists an instance $\nu \in \cE$ where no algorithm can distinguish which arm has a higher quantile using a finite number of arm pulls (see \cite[Theorem 2]{nikolakakis2021quantile}), but Algorithm~\ref{alg: main} is capable of returning an $\epsilon$-satisfying arm with high probability.

\begin{proposition}
\label{prop: converse zero gap not true}
    Fix $\lambda \ge \epsilon > 0$ and $\delta \in (0, 0.5)$.
    There exists a two-arm bandit instance $\nu \in \cE$ that has a unique arm $k^*$ with the highest median such that
    $\Delta =  \Delta(\nu, \lambda, \epsilon, c, q) $ as defined in Theorem~\ref{theorem: upper bound} is positive for $c \ge 2$,
    but $\min\limits_{a \in \A} \big\{ \Delta_a^{\mathrm{NKSS}}  \big\} 
= \min\limits_{a \in \A} \big\{ \Delta_a^{\mathrm{HR}}  \big\} = 0$.
\end{proposition}

\begin{proof}
    Consider two arms $\A = \{1, 2\}$ with the following CDFs:
\begin{equation}
    F_1(x) = 
    \begin{cases}
        0  & \text{ for } x < 0
        \\
        \frac{x}{2m_1} & \text{ for } 0 \le x < 2m_1 \\
        1 & \text{ for } x \ge 2m_1
    \end{cases}
    \quad \text{and} \quad
    F_2(x) = 
    \begin{cases}
        0 & \text{ for } x < m_2 \\
        0.5 & \text{ for } m_2 \le x < 2m_1 \\
        1 & \text{ for } x \ge 2 m_1
    \end{cases},
\end{equation}
where $ m_2 \in (m_1 - \epsilon/2, m_1)$
such that both arms are $\epsilon$-optimal, with arm 1 being the unique best arm. 
Note that for each $\eta > 0$, we have
\begin{equation}
    Q_2(0.5 + \eta) = 2 m_1
    > m_1 = Q_1(0.5) \ge Q_1(0.5 - \eta),
\end{equation}
and so $\Delta_2^{\mathrm{NKSS}} = \Delta_2^{\mathrm{HR}} = 0$.
However, under our gap definition (Definition~\ref{def: our gap}) with $\A_{\epsilon}(\nu) = \{1, 2\} = \A$ and any $c \ge 2$, we have
\begin{align}
    \Delta \ge \Delta_2 
    \ge \Delta_{2}^{(\{1,2\})}
    &=
    \sup
    \left\{
        \Delta \in [0, 0.5]
        :
        Q^+_2(0.5 - \Delta) 
        \ge
        Q_{1}(0.5 + \Delta) - c\tilde{\epsilon}
        \right\} \\
    &\ge
    \sup
    \left\{
        \Delta  \in [0, 0.5]
        :
        Q^+_2(0.5 - \Delta) 
        \ge
        Q_{1}(0.5 + \Delta) - \frac{\epsilon}{2}
        \right\} \\    
    &=
    \sup
    \left\{
        \Delta  \in [0, 0.5]
        :
        m_2
        \ge
        (1+2 \Delta) m_1 - \frac{\epsilon}{2}
        \right\} \\
     &= \min\left\{0.5, \frac{m_2 - (m_1 - \epsilon/2)}{2m_1} \right\} >0,
\end{align}
where the second inequality follows from the calculation in Remark~\ref{rem: picking large enough c}, and the last inequality follows from the assumptions that $m_1 > 0$ and $m_2 > m_1 - \epsilon/2$.
\end{proof}

\section{Proof of Theorem~\ref{theorem: upper bound} (Upper Bound of Algorithm~\ref{alg: main})}
\label{sec: appendix upper bound}
We break down the upper bound on the number of arm pulls used by Algorithm~\ref{alg: main} as follows. We bound the number of rounds required for a non-satisfying arm $k \not\in \A_{\epsilon}(\nu)$ to be eliminated in Lemma~\ref{lem: elim suboptimal}. Then in Lemma~\ref{lem: termination}, we bound the number of rounds each non-eliminated arm has gone through when the termination condition of the while-loop is triggered. Combining these lemmas with the number of arm pulls used by $\mathtt{QuantEst}$ for each round index $t \ge 1$ and active arm $k \in \A_t$ as stated in~\eqref{eq: QuantEst arm pulls}
gives us an upper bound on the total number of arm pulls.

We first present a useful lemma that will be used in the proofs of the two subsequent lemmas. 

\begin{lemma}[$\max \mathrm{LCB}$ is non-decreasing]
\label{lem: max LCB increasing}
    Under Event $E$ as defined in Lemma~\ref{lem: good events}, we have 
    \begin{equation}
    \max\limits_{a \in \mathcal{A}_{t}} 
            \mathrm{LCB}_{t}(a) \ge 
    \max\limits_{a \in \mathcal{A}_{\tau}} 
            \mathrm{LCB}_{\tau}(a).
    \end{equation}
    for all rounds $t > \tau  \ge 1$.
\end{lemma}
\begin{proof}
    Let round index $\tau \ge 1$ be arbitrary
    and let $k \in \argmax\limits_{a \in \A_{\tau}} 
            \mathrm{LCB}_{\tau}(a).$
    We have $k \in \A_{\tau+1}$ since 
    $\mathrm{UCB}_{\tau}(k) > \mathrm{LCB}_{\tau}(k) 
    = \max\limits_{a \in \A_{\tau}} 
            \mathrm{LCB}_{\tau}(a)$ by~\eqref{eq:  quantile anytime bound} 
    of the anytime quantile bounds.        
    It then follows that
    \begin{equation}
        \max\limits_{a \in \mathcal{A}_{\tau+1}} 
            \mathrm{LCB}_{\tau+1}(a) 
        \ge \mathrm{LCB}_{\tau+1}(k) 
        \ge \mathrm{LCB}_{\tau}(k) = 
        \max\limits_{a \in \A_{\tau}} 
            \mathrm{LCB}_{\tau}(j),
    \end{equation}    
    where the second inequality follows from~\eqref{eq:  quantile anytime bound} 
    of the anytime quantile bounds. 
    Applying the argument repeatedly yields the claim for all $t > \tau.$
\end{proof}

\begin{lemma}[Elimination of non-satisfying arms]
\label{lem: elim suboptimal}
     Fix an instance $\nu \in \cE$, and suppose Algorithm~\ref{alg: main} is run with input $(\A, \lambda, \epsilon, q, \delta)$ and parameter $c \ge 1$.
    Let $\A_{\epsilon} = \A_{\epsilon}(\nu) $ be as defined in~\eqref{def: performance def} and let the gap $\Delta_{k} = \Delta_{k}(\nu, \lambda, \epsilon, c, q)$ be as defined in Definition~\ref{def: our gap} 
    for each arm $k \in \A$.
    Consider an arm $k \not\in \A_{\epsilon}$.
    Under Event $E$ as defined in Lemma~\ref{lem: good events}, when the round index~$t$
    of Algorithm~\ref{alg: main} satisfies $\Delta^{(t)}  \le \frac{1}{2} \Delta_k$, we have  $k \not\in \A_{t+1}$.
\end{lemma}
\begin{proof}
    If $k \not\in \A_t$, then $k \not\in \A_{t+1}$ trivially. Therefore, we assume for the rest of the proof that $k \in \A_t$, and we will show that
    \begin{equation}
    \label{eq: eliminate condition}
        \mathrm{UCB}_t(k) \le \max\limits_{a \in \mathcal{A}_{t}} \mathrm{LCB}_t(a)
    \end{equation}
    or equivalently
    \begin{equation}
    \label{eq: eliminate condition equivalent}
        \mathrm{UCB}_t(k) < \max\limits_{a \in \mathcal{A}_{t}} \mathrm{LCB}_t(a) + \tilde{\epsilon},
    \end{equation}
    where $\tilde{\epsilon} = \tilde{\epsilon}(\lambda, \epsilon, c)$ is as defined in Lines~\ref{line: number of points} and~\ref{line: tilde epsilon} of Algorithm~\ref{alg: main}.
    Note that these conditions are equivalent because both
    $\mathrm{UCB}_t(k)$ and
    $\max\limits_{a \in \A_t } \mathrm{LCB}_t(a)$ 
    are elements of 
    \begin{equation}
        \left[ 0, 
        \tilde{\epsilon}, 
        2\tilde{\epsilon}, \cdots,
        (n-1) \tilde{\epsilon}, \lambda\right],
    \end{equation}
   which follows from Lines~\ref{line: list of points},~\ref{eq: initiate default conf interval}, and~\ref{ltk def}--\ref{UCB definition} of Algorithm~\ref{alg: main}.

    Since $k \not\in \A_{\epsilon}$, when the round index $t$ satisfies~$\Delta^{(t)} \le \frac{1}{2} \Delta_k $ we have
    \begin{equation}
    \label{eq: gap k realized with arm j}
        \mathrm{UCB}_t(k)
        < Q_k \big( q + \Delta^{(t)} \big)  + \tilde{\epsilon} 
        \le Q^+_{j}\big(q - \Delta^{(t)} \big) 
    \end{equation}
    for some arm $j \in \A$  by~\eqref{eq: upper approx quantile anytime bound} of the anytime quantile bounds
    and Definition~\ref{def: our gap}. 
    We now consider two cases: (i) $j \in \A_t$ and (ii) $j \not\in \A_t$.

    If $j \in \A_t$, we have
    \begin{equation}
    \label{eq: j in At}
        Q^+_{j}\big(q - \Delta^{(t)}\big) 
        \le \mathrm{LCB}_t(j) + \tilde{\epsilon} 
        \le \max\limits_{a \in \mathcal{A}_{t}} \mathrm{LCB}_t(a) + \tilde{\epsilon} 
    \end{equation}
    by~\eqref{eq: lower approx quantile anytime bound} of the anytime quantile bounds and the assumption that $j \in \A_t$. Combining~\eqref{eq: gap k realized with arm j} and~\eqref{eq: j in At} gives us condition~\eqref{eq: eliminate condition equivalent} as desired.
    
    If $j \not\in \A_t$, then it is eliminated at some round $\tau < t$, i.e., 
    $j \in \A_{\tau}$ but $j \not\in \A_{\tau + 1}$.
    By \eqref{eq: quantile anytime bound} of the anytime quantile bounds, the definition of active arm set (Line~\ref{line: active arm} of Algorithm~\ref{alg: main}) applied to $\A_{\tau + 1}$,
    and the fact that max LCB is non-decreasing (Lemma~\ref{lem: max LCB increasing}), 
    we have 
    \begin{equation}
    \label{eq: j not in At}
    Q_{j}(q) 
    \le
    \mathrm{UCB}_{\tau}(j) 
    \le
            \max\limits_{a \in \mathcal{A}_{\tau}} 
            \mathrm{LCB}_{\tau}(a)     
        \le
        \max\limits_{a \in \mathcal{A}_{t}} 
            \mathrm{LCB}_{t}(a). 
    \end{equation}
    Combining~\eqref{eq: gap k realized with arm j}, the trivial inequality 
    $Q^+_{j}\big(q - \Delta^{(t)}\big) 
        \le Q_{j}(q) $, and~\eqref{eq: j not in At} yields
    condition~\eqref{eq: eliminate condition} as desired.
\end{proof}

\begin{remark}
    \label{rem: elim suboptimal}
    As seen in the analysis for the case  $j \not\in \A_t$ above, the property that $\max \mathrm{LCB}$ is non-decreasing (Lemma~\ref{lem: max LCB increasing}) is crucial in establishing~\eqref{eq: j not in At}. We will see below that the same argument is used again in establishing~\eqref{eq: LCB_t(k) > Fj}. This property of Lemma~\ref{lem: max LCB increasing} itself is a consequence of ensuring $\mathrm{LCB}_t(k)$ is non-decreasing in $t$; see Remark~\ref{rem: LCB non decreasing}.
\end{remark}

\begin{lemma}[While-loop termination]
\label{lem: termination}
     Fix an instance $\nu \in \cE$, and suppose Algorithm~\ref{alg: main} is run with input $(\A, \lambda, \epsilon, q, \delta)$ and parameter $c \ge 1$.
    Let $\A_{\epsilon} = \A_{\epsilon}(\nu) $ be as defined in~\eqref{def: performance def} and let the gap $\Delta_{k} = \Delta_{k}(\nu, \lambda, \epsilon, c, q)$ be as defined in Definition~\ref{def: our gap} 
    for each arm $k \in \A$.
    Under Event $E$, when the round index~$t$
    of Algorithm~\ref{alg: main} satisfies $\Delta^{(t)} \le \frac{1}{2} \max \limits_{a \in \A_{\epsilon}} \Delta_a$, Algorithm~\ref{alg: main} will terminate in round $t+1$.
\end{lemma}

\begin{proof}
     If $\A_{t+1} = \{ k^* \}$, then
      \begin{equation}
          \max\limits_{a \in \A_{t+1} \setminus \{k^*\} }                 \mathrm{UCB}_t(a) - (c+1)\tilde{\epsilon}
      = -\infty \le \mathrm{LCB}_{t}(k^*),
      \end{equation}
     and so the algorithm will terminate and return arm $k^*$ in round $t+1$.
     Therefore, we assume for the rest of the proof that
     there exists another arm $a \ne k^*$ such that $a \in \A_{t+1}$. 

     We first show that the following condition is sufficient to trigger the termination condition of the while-loop (Lines~\ref{line: start while loop}--\ref{line: end while loop}) of Algorithm~\ref{alg: main}: There exists an arm $k \in \A_{t+1}$
     satisfying
    \begin{equation}
    \label{eq: suf cond trigger termination}
          \mathrm{LCB}_t(k)  
          \ge
          \max\limits_{a \in \A_{t+1} \setminus \{k\} }
        Q_{a}\big(q + \Delta^{(t)}\big) -  (c+1)\tilde{\epsilon} .
    \end{equation}
    Using~\eqref{eq: upper approx quantile anytime bound} of the anytime quantile bound, 
    condition~\eqref{eq: suf cond trigger termination} implies that
    \begin{equation}
    \label{eq: termination condition strict equality}
        \mathrm{LCB}_t(k)  
        >
          \max\limits_{a \in \mathcal{A}_{t+1} \setminus \{k\} } \mathrm{UCB}_t(a)
          - (c+2)\tilde{\epsilon},
    \end{equation}
    which is equivalent to the termination condition
    \begin{equation}
          \mathrm{LCB}_t(k)  
            \ge
          \max\limits_{a \in \mathcal{A}_{t+1} \setminus \{k\} } \mathrm{UCB}_t(a) - (c+1)\tilde{\epsilon},
    \end{equation}
    where the equivalence follows from an argument similar to the equivalence between~\eqref{eq: eliminate condition} and 
        \eqref{eq: eliminate condition equivalent}.

    It remains to pick an arm $k \in \A_{t+1}$ satisfying condition~\eqref{eq: suf cond trigger termination}.
    Let arm $j \in \argmax\limits_{a \in \A_{\epsilon}} \Delta_a$ and consider the following two cases: (i) $j \in \A_{t+1}$ and (ii) $j \not\in \A_{t+1}$.

    If $j \in \A_{t+1}$, we pick $k = j$. We also pick~$
    T \in \argmax\limits_{\A_{\epsilon} \subseteq S \subseteq \A}
        \Delta_{k}^{(S)}    
    $ 
    to be the set associated to $\Delta_k$ (see Definition~\ref{def: our gap}).
    Note that every arm that is not in $T$
    is a non-satisfying arm since 
    $\A_{\epsilon} \subseteq T$.
    Furthermore, every non-satisfying arm that is not in $T$, hence every arm that is not in $T$, is eliminated,
    which follows from Lemma~\ref{lem: elim suboptimal} and
    \begin{equation}
        \Delta^{(t)} 
        \le \frac{1}{2} \max \limits_{a \in \A_{\epsilon}} \Delta_a
        = 
        \frac{1}{2} \Delta_k \le \frac{1}{2} \min\limits_{a \not\in T} \Delta_a,
    \end{equation} 
    where the last inequality follows from applying~\eqref{eq: Delta k^S} to $k$ and $T$.
    Therefore, we have
    $\A_{t+1} \subseteq T$.
    It follows that
    \begin{align}
         \mathrm{LCB}_t(k)
          &\ge
          Q^+_k\big(q - \Delta^{(t)}\big) - \tilde{\epsilon} \\
          &\ge
        \max\limits_{a \in T \setminus \{k\} }
        Q_{a}\big(q + \Delta^{(t)}\big) - (c+1)\tilde{\epsilon} \\
        &\ge
      \max\limits_{a \in \A_{t+1} \setminus \{k\} }
        Q_{a}\big(q + \Delta^{(t)}\big) - (c+1)\tilde{\epsilon},
    \end{align}
    where the first inequality follows from~\eqref{eq: lower approx quantile anytime bound} of the anytime quantile bound, the second inequality follows from applying~\eqref{eq: Delta k^S} to $k$ and $T$,
    and the last inequality follows from  $\A_{t+1} \subseteq T$.

    If $j \not\in \A_{t+1}$,
    we pick an arm 
    $k \in \argmax\limits_{a \in \mathcal{A}_{t+1}} 
    \mathrm{LCB}_{t}(a)$ arbitrarily.
    We also pick $T \in \argmax\limits_{\A_{\epsilon} \subseteq S \subseteq \A} \Delta_{k}^{(S)}$ and  
    we have $\A_{t+1} \subseteq T$ as in the case above.
    Furthermore, since $j \not\in \A_{t+1}$,
    we have
    \begin{equation}
    \label{eq: LCB_t(k) > Fj}
        Q^+_j \big(q - \Delta^{(t)}\big)
    \le 
    Q_{j}(q) 
    \le
    \max\limits_{a \in \mathcal{A}_{t+1}} 
    \mathrm{LCB}_{t}(a) 
    = \mathrm{LCB}_{t}(k),
    \end{equation}
    where the second inequality follows from an argument similar to~\eqref{eq: j not in At}.
    It follows that
    \begin{align}
         \mathrm{LCB}_t(k) 
        &\ge Q^+_j \big(q - \Delta^{(t)}\big)  \\
        &\ge
        \max\limits_{a \in T \setminus \{j\} }
        Q_{a}\big(q + \Delta^{(t)}\big)   - c \tilde{\epsilon} \\
        &\ge
        \max\limits_{a \in \A_{t+1} \setminus \{k\} }
        Q_{a}\big(q + \Delta^{(t)}\big) - (c+1) \tilde{\epsilon},
    \end{align}
    where the first inequality follows from~\eqref{eq: LCB_t(k) > Fj}, the second inequality follows from applying~\eqref{eq: Delta k^S} to $j$ and $T$,
    and the last inequality follows from  $\A_{t+1} \subseteq T$.
\end{proof}

\section{Lower Bounds}
\label{sec: appendix lower bound}

\subsection{Proof of Theorem~\ref{thm: lower bound unquantized} (Lower Bound for the Unquantized Variant)}
\label{sec: appendix unquantized lower bound}

Since we are adapting the instance from~\cite[Theorem  4]{nikolakakis2021quantile}, we will omit certain details for brevity and instead will focus on the main differences.

\begin{proof}[Proof of Theorem~\ref{thm: lower bound unquantized}]
    Define the following class of distributions parametrized by $w \in (0, q)$:
\begin{equation}
\label{eq: mix Dirac and uniform}
    g_{w}(x) \coloneqq
   w \delta(x) + 1-w,
\end{equation}
i.e., $g_w$ is a mixture of the Dirac delta function and a uniform distribution on $[0, 1]$. 
Fix $w,\gamma \in (0, q)$ such that $w+\gamma  \le q $. Note that $g_w$ has a higher $q$-quantile than $g_{w+\gamma}$ since
\begin{equation}
\label{eq: diff of arms}
    G^{-1}_w(q) -
    G^{-1}_{w+\gamma}(q) = 
    \frac{q - w}{1-w} - 
    \frac{q - (w+\gamma)}{1-(w+\gamma)} =
    \frac{(1-q)\gamma}{(1-w)(1-w-\gamma)}> 0,
\end{equation}
where $G_{w}^{-1}$ is the lower quantile function of $g_w$.

We now use~\eqref{eq: mix Dirac and uniform} to define a set of $K$ instances $\{\nu^{(1)}, \ldots, \nu^{(K)} \} \subseteq \cE$ for our QMAB problem. Here, 
each $\nu^{(j)}$ is a different instance of the arm distributions, with $\nu^{(j)}_k$ being the CDF of arm $k$ for instance~$j$.
Fix $\gamma \in (0,  1/6]$.
For $\nu^{(1)}$, we define the arms' PDF by
\begin{equation}
\label{eq: bad instance}
    \nu^{(1)}_k \coloneqq
    \begin{cases}
        g_{1/3- \gamma} & \text{ if } k = 1 \\
        g_{1/3} & \text{ if } k \ne 1.
    \end{cases}
\end{equation}
For $j = 2, \ldots, K$, we define the arms' PDF of $\nu^{(j)}$ by
\begin{equation}
    \nu^{(j)}_k  \coloneqq
    \begin{cases}
        g_{1/3- \gamma} & \text{ if } k = 1 \\
        g_{1/3- 2\gamma} & \text{ if } k = j \\
        g_{1/3} & \text{ if } k \ne 1 \text{ or } j.
    \end{cases}
\end{equation}
We will use $\nu^{(1)}$ as the ``hard instance" in our lower bound.
By assumption of our $\epsilon$, we have arm 1 being the unique satisfying arm for $\nu^{(1)}$. Using~\eqref{eq: epsilon condition}) and~\eqref{eq: diff of arms}, we have
\begin{equation}
\label{eq: epsilon upper bound}
    \epsilon  \le
    \frac{1}{2} \Big( Q^{(1)}_{k^*}(q) - \max \limits_{k \ne k^*} Q^{(1)}_k(q) \Big)
    = \frac{G^{-1}_{1/3 -\gamma}(q) - G^{-1}_{1/3}(q) }{2 } = 
    \frac{(1-q)\gamma}{2(2/3 + \gamma )(2/3)}.
\end{equation}
This implies that arm $j$ is the unique satisfying arm for $\nu^{(j)}$ for $j = 2, \ldots, K$ since
\begin{equation}
    G^{-1}_{1/3 -2\gamma}(q) - G^{-1}_{1/3-\gamma}(q)  = 
    \frac{(1-q)\gamma}{(2/3 + 2\gamma )(2/3 + \gamma)}
    =
    \frac{(1-q)\gamma}{2(2/3 + \gamma )(2/3)} \cdot
    \frac{4/3}{2/3 + 2\gamma}
    \ge \epsilon,
\end{equation}
where the inequality follows from~\eqref{eq: epsilon upper bound} and $\gamma \le 1/6$.

To establish the lower bound on the arm pulls for instance $\nu^{(1)}$, we  first upper bound the inverse arm gap $\Delta_j^{-1}$ in terms of $\gamma$ for the arms in $\nu^{(1)}$. For arm 1 and each non-satisfying arm $j \ne 1$, we have
\begin{align}
    \Delta_1 \ge    
    \Delta_j=&\sup
    \left\{
        \Delta \in \left[0, \min(q, 1-q) \right]
        \colon
        G^{-1}_{1/3}(q + \Delta) 
        \le
        G^{-1}_{1/3-\gamma}(q - \Delta) 
        \right\} 
         \label{eq: gap g_gamma first line}
        \\
    =&
    \min\left\{
    \sup
    \left\{
        \Delta \ge 0
        \colon
         \left(\frac{q  +  \Delta  - 1/3}{2/3} \right)
         \le
         \left(\frac{q  - \Delta  - 1/3 +\gamma }{2/3+\gamma }\right)  
        \right\} , 
         q, 1-q
        \right\} \\
      =&   
      \min\left\{\frac{ (1-q)\gamma}
      {(4/3 + \gamma )}, q, 1-q \right\} 
       \label{eq: gap g_gamma last line} \\
      =&   
     \frac{ (1-q)\gamma}
      {(4/3 + \gamma )} \\
    \ge&
    \frac{2(1-q)\gamma}{3},
    \label{eq: delta_k gamma}
\end{align}
where the first inequality follows from the argument in~\eqref{eq: lower bound k* arm gap} and the second inequality follows from $\gamma \le 1/6$.

Fix an $(\epsilon, \delta)$-reliable algorithm $\pi$ (see Definition~\ref{def: reliable}), and let $\tau \le \infty$ be the total number of arm pulls by $\pi$ on instance $\nu^{(1)}$. We may assume that $\PP_{\nu^{(1)}}[\tau = \infty] = 0$, since otherwise $\E_{\nu^{(1)}}[\tau] = \infty$ and the theorem holds trivially.
For each $j \in \{2, \ldots, K\}$, define event $A_j$ to be 
\begin{equation}
    A_j \coloneqq \{ \pi \text{ terminates and outputs $\hat{k} \ne j$} \}.
\end{equation}
By the definition of $(\epsilon, \delta)$-reliability, we must have
\begin{equation}
    \mathbb{P}_{\nu^{(j)}}\big[A_j\big]
    \le \delta \text{ for each } j \in \{2, \ldots, K\}
    \quad \text{and} \quad
    \mathbb{P}_{\nu^{(1)}}\big[ \tau < \infty \cap   \hat{k} \ne 1 \big] \le \delta.
\end{equation}
    Using the assumption $\PP_{\nu^{(1)}}[\tau = \infty] = 0$ 
    and the event inclusion $\{\hat{k} = j\} \subseteq \{\hat{k} \ne 1\}$, we have
\begin{equation}
    \mathbb{P}_{\nu^{(1)}}\big[A_j^{\complement} \big] =
    \mathbb{P}_{\nu^{(1)}}\big[\tau = \infty \cup \hat{k} = j \big] =
    \mathbb{P}_{\nu^{(1)}}\big[\hat{k} = j \big]
    \le
    \mathbb{P}_{\nu^{(1)}}\big[\hat{k} \ne 1 \big] = 
    \mathbb{P}_{\nu^{(1)}}\big[ \tau < \infty \cap   \hat{k} \ne 1 \big] \le \delta,
\end{equation}
and so 
\begin{equation}
    \mathbb{P}_{\nu^{(1)}}\big[A_j^{\complement} \big] + 
 \mathbb{P}_{\nu^{(j)}}\big[A_j  \big] \le 2 \delta \quad 
 \text{ for each } j \in \{2, \ldots, K\}.
\end{equation}
Let $T_j \le \tau$ be the number of times arm $j$ is pulled on $\nu^{(1)}$. For a fixed $j \in \{2, \ldots, K\}$, we have
\begin{equation}
    \E_{\nu^{(1)}}[T_j]
    \ge 
    \frac{D_{\mathrm{KL}}
    \left(\mathbb{P}_{\nu^{(1)}} \parallel  \mathbb{P}_{\nu^{(j)}}\right)}{12 \gamma^2} 
    \ge
    \frac{1}{12 \gamma^2} \log \left( \frac{1}{4 \delta} \right) 
\end{equation}
where the inequalities follow from \cite[Eqn. 29--34]{nikolakakis2021quantile}.
Summing through $j  = 2, \ldots, K$
and 
we have
\begin{align}
    \E_{\nu^{(1)}}[\tau]
    \ge
    \sum_{j=2}^K \E_{\nu^{(1)}}[T_j] 
    \ge
    \sum_{j=2}^K  
    \frac{1}{12 \gamma^2} \log \left( \frac{1}{4 \delta} \right) 
     \ge
    \frac{1}{2}
    \sum_{j=1}^K
     \frac{1}{12 \gamma^2} \log \left( \frac{1}{4 \delta}  \right),
\end{align}
where the last inequality follows from $K \ge 2$.
Applying the bounds~\eqref{eq: gap g_gamma first line}--\eqref{eq: delta_k gamma} for each $j$ yields
\begin{equation}
    \E_{\nu^{(1)}}[\tau]
     \ge
     \sum_{j=1}^K
    \frac{(1-q)^2}{27 \Delta_j^2}
    \log \left( \frac{1}{4 \delta} \right)
    = \Omega \bigg(
\sum_{k=1}^K
 \frac{1}{\Delta_k^2}
    \log \left( \frac{1}{\delta} \right)
    \bigg),
\end{equation}
as desired.
\end{proof}

\subsection{Proof of Theorem~\ref{thm: log lambda/epsilon dependence} ($\Omega(\log(\lambda/\epsilon))$ Dependence)}
\label{sec: appendix log lambda epsilon dependence}
\begin{proof}[Proof of Theorem~\ref{thm: log lambda/epsilon dependence}]
    Let $\nu$ be a two-arm QMAB instance with deterministic but unknown $q$-quantile rewards $r_1$ and $r_2$
satisfying\footnote{For ease of analysis, we assume $\lambda$ is an integer multiple of $2 \epsilon$.}
\begin{equation}
\label{eq: r1 r2 assumption}
    r_1, r_2 \in  
    \left[0, 2\epsilon, 4 \epsilon, \ldots, \lambda \right]
    \quad
    \text{and}
    \quad
     |r_1- r_2| = 2 \epsilon.
\end{equation}
In this case, the only arm satisfying \eqref{def: performance def} is the one with the higher $q$-quantile. Since the rewards are deterministic, the QMAB problem is equivalent to finding out which of $r_1$ and $r_2$ is higher.

We consider a modified threshold query setup where the learner receives more information at each iteration: At iteration $t$, the learner decides a threshold $X_t \in [0, \lambda]$, and receives a 2-bit comparison feedback in the form of $(\mathbf{1}(r_1 \le X_t), \mathbf{1}(r_2 \le X_t))$.
By design, the number of iterations
required under the 2-bit threshold query setup is at most the number of arm pulls required under the 1-bit threshold query setup.

We now establish the lower bound of $\Omega(\log(\lambda/\epsilon))$ on the number of iterations needed to determine which of $r_1$ and $r_2$ is higher for instance $\nu$ under the 2-bit threshold query setup.
We first claim that for an algorithm to be $(\epsilon, \delta)$-reliable, the learner has to keep querying until receiving some feedback satisfying
\begin{equation}
    \label{eq: two det arm stopping condition}
    (\mathbf{1}(r_1 \le X_t), \mathbf{1}(r_2 \le X_t)) \in \{ (0, 1) ,(1, 0) \},
\end{equation}
which occurs if and only if 
$X_t \in [\min(r_1, r_2), \max(r_1, r_2)]$.
Feedback of the form in~\eqref{eq: two det arm stopping condition} 
is necessary as otherwise instance $\nu$ is indistinguishable from instance $\nu'$ where $r_1$ and $r_2$ are swapped, and the best any algorithm could do is to make a 50/50 guess, which is not $(\epsilon, \delta)$-reliable for $\delta < 0.5$.

To establish the lower bound, we may assume that the learner knows that~\eqref{eq: r1 r2 assumption} holds, since extra information can only weaken a lower bound. With this information, instead of picking $X_t$ from the interval $[0, \lambda]$, the learner could pick $X_t$ only from the list~$X \coloneqq \left[0, 2\epsilon, 4 \epsilon, \ldots, \lambda\right]$ without loss of generality (any other choices would have a corresponding equivalent choice in this set).  
As there is exactly one $x \in X$ that would lead to feedback of the form in~\eqref{eq: two det arm stopping condition}, we need to identify one of $|X|$ possible outcomes, which amounts to learning $\log_2 |X|$ bits.
Since each threshold query gives a 2-bit feedback, the number of threshold queries/iterations needed in the worst case is $\Omega(\log(|X|)) =  \Omega(\log(\lambda/\epsilon))$.
\end{proof}

\section{Proof of Theorem \ref{thm: zero gap is unsolvable} and Corollary~\ref{cor: zero gap is unsolvable} (Solvable Instances)}
\label{sec: appendix solvable instance}

We first state a useful lemma for Theorem \ref{thm: zero gap is unsolvable}.

\begin{lemma} 
\label{lem:two_instances}
        Let $\lambda, \epsilon, c,$ and $q$ be given, 
        and let $\tilde{\epsilon} = \tilde{\epsilon}(\lambda, \epsilon, c)$ be as defined in 
        Algorithm~\ref{alg: main}.
        Suppose that $\nu \in \cE$ is an instance with gap $\Delta(\nu, \lambda, \epsilon, c, q) = 0 $ and let $\eta_0 = \eta_0(\nu) > 0$ be the constant given in the assumption in Theorem \ref{thm: zero gap is unsolvable}.  
        Then, for each arm $k \in \A_{c \tilde{\epsilon}}(\nu)$ and each $\eta \in  (0, \eta_0) $, there exists another instance $\nu' \in \cE$ satisfying the following:
    \begin{itemize}[topsep=0pt, itemsep=0pt]
        \item There exists an arm $a \in \Ac \setminus \{k\}$ such that instances $\nu$ and $\nu'$ are identical for all arms in $\Ac \setminus \{a,k\}$;
        
        \item $\dTV(F_a,G_a) \le \eta$ and $\dTV(F_k,G_k) \le \eta$, where $F_{(\cdot)}$ and $G_{(\cdot)}$ represent the arm distributions for instances $\nu$ and $\nu'$ respectively;
        
        \item $k \notin \Ac_{c \tilde{\epsilon} }(\nu')$, i.e., 
        under relaxation parameter $c \tilde{\epsilon}$, arm $k$ is not a satisfying arm for instance $\nu'$.
    \end{itemize}
\end{lemma}

 \begin{proof}   
    Let $\nu \in \cE$ be an instance with gap $\Delta(\nu, \lambda, \epsilon, c, q) = 0$. For each arm $k \in \A_{\epsilon}(\nu)$, we have $\Delta_{k}^{(\A)} = 0$ by Definition~\ref{def: our gap} since 
    $0 \le \Delta_{k}^{(\A)}  \le \Delta_{k}  \le \Delta  = 0$.
    Applying \eqref{eq: Delta k^S} with set $S = \A$ yields:
     \begin{equation}
     \label{eq: arm a positive eta}
        \text{for each } k \in \A_{\epsilon}(\nu)
        \text{ and each } \eta > 0, 
        \text{ there exists } a \ne k
        \text{ such that }
         Q^+_{k}(q - \eta) 
        <
        Q_{a}(q + \eta) - c\tilde{\epsilon}.
     \end{equation}
    Fix an arm $k \in \A_{c \tilde{\epsilon}}(\nu)$ and $\eta \in (0, \eta_0)$.
    Since $c \tilde{\epsilon} \le \epsilon$ (see calculation in~\eqref{eq: tilde eps 1 and 2}--\eqref{eq: c1 tilde eps 1 and 2}), we have 
    $\A_{c \tilde{\epsilon}}(\nu) \subseteq \A_{\epsilon}(\nu)$, and hence $k \in \A_{\epsilon}(\nu)$. 
    It follows from~\eqref{eq: arm a positive eta} that there exists some arm $a \ne k$ that
    \begin{equation}
     \label{eq: arm a positive eta c tilde epsilon}
         Q^+_{k}(q - \eta) 
        <
         Q_{a}(q + \eta) - c\tilde{\epsilon}.
     \end{equation}
    We now construct instance $\nu'$ such that $\nu$ and $\nu'$
    have identical distributions for all arms in $\A \setminus \{a, k\}$, 
    while $F_a$ and $F_k$ are being replaced with $G_a$ and $G_k$ defined as follows:
    \begin{enumerate}[topsep=0pt, itemsep=0pt]
        \item 
        $G_a$ is any distribution obtained by moving $\eta$-probability mass from the interval $(-\infty, Q_a(q))$ to the point $Q_a(q+2\eta)$; 
        
        \item 
        $G_k$ is any distribution obtained by moving $\eta$-probability mass from the interval $(Q_k(q), \infty)$ to the point $ Q_k(q-2\eta)$.
    \end{enumerate}     
     Under these definitions and the assumption on $\eta_0$ in Theorem~\ref{thm: zero gap is unsolvable}, we can readily verify that
     \begin{equation}
     \label{eq: shifted q quantiles}
         (G_k)^{-1}(q) =  Q_k(q-\eta) 
         \in [0, \lambda]
         \quad 
         \text{and}
         \quad  
         (G_a)^{-1}(q) = Q_a(q+\eta) \in   [0, \lambda]
     \end{equation}
     and
     \begin{equation}
         d_{\mathrm{TV}}(F_k, G_k) =  d_{\mathrm{TV}}(F_a, G_a) = \eta.
     \end{equation}
    
     Finally, combining~\eqref{eq: arm a positive eta c tilde epsilon} and~\eqref{eq: shifted q quantiles} yields
     \begin{equation}
     \label{eq: G_k unsatisfying}
          (G_k)^{-1}(q) 
         < 
         (G_a)^{-1}(q) - c\tilde{\epsilon},
     \end{equation}
     which implies $k \notin \Ac_{c \tilde{\epsilon} }(\nu')$. 
     By construction, $\nu'$ satisfies all three properties as desired.
\end{proof}

\begin{remark}
\label{rem: limit version of two instance lemma}
     We can obtain a ``limiting'' version of Lemma~\ref{lem:two_instances} in which we replace the gap $\Delta(\nu, \lambda, \epsilon, c, q)$ by $\Delta(\nu, \epsilon, q)$ as defined in Corollary~\ref{cor: zero gap is unsolvable} and the satisfying arm set $\A_{c \tilde{\epsilon}}(\cdot)$
     by $\A_{\epsilon}(\cdot)$.
     The proof is essentially identical.
     We construct instance $\nu'$ in a similar manner as above to satisfy the first two properties in the statement of Lemma~\ref{lem:two_instances}.
     The last property $(k \not\in \A_{\epsilon}(\nu'))$ then follows from the definition of the limit gap $\Delta_{k}(\nu, \epsilon, 
    q)$ as defined in~\eqref{eq: gap k infinite c}, which allows us to replace the $c\tilde{\epsilon}$ terms in~\eqref{eq: arm a positive eta},~\eqref{eq: arm a positive eta c tilde epsilon}, and~\eqref{eq: G_k unsatisfying} by $\epsilon$.    
\end{remark}

We proceed to prove Theorem \ref{thm: zero gap is unsolvable}.  
\begin{proof}[Proof of Theorem~\ref{thm: zero gap is unsolvable}]
Assume for contradiction that there exists some instance $\nu \in \cE$ satisfies $\Delta(\nu, \lambda, \epsilon, c, q) = 0 $, but is $c\tilde{\epsilon}$-solvable. 
Fix a $\delta \in (0, 1)$ satisfying
\begin{equation}
    \label{eq: delta very small}
    \delta < \frac{1}{2+2|\A|}.
\end{equation}
By Definition \ref{def:solvable}, there exists a $(c\tilde{\epsilon}, \delta)$-reliable algorithm such that
\begin{equation}
    \PP_{\nu}[\tau < \infty \cap \hat{k} \in \Ac_{c\tilde{\epsilon}}(\nu)] \ge 1-\delta.
\end{equation}
In general the condition $\tau < \infty$ may not imply a \emph{uniform} upper bound on $\tau$; we handle this by relaxing the probability from $1-\delta$ to $1-2\delta$, such that there exists some $\tau_{\max} < \infty$ satisfying
\begin{equation}
    \PP_{\nu}[\hat{k} \in \Ac_{c\tilde{\epsilon}}(\nu) \cap \tau \le \tau_{\max}] \ge 1-2\delta. \label{eq:tau_max}
\end{equation}
From this, we claim that there exists an arm $k_{\nu} \in \Ac_{c\tilde{\epsilon}}(\nu)$ such that
\begin{equation}
    \PP_{\nu}[\hat{k} = k_{\nu} \cap \tau \le \tau_{\max}] \ge \frac{1-2\delta}{|\Ac|}. 
    \label{eq:success_nu}
\end{equation}
Indeed, if this were not the case, then summing these probabilities over elements in $\Ac_{\epsilon}(\nu)$ would produce a total below $1-2\delta$, which would contradict \eqref{eq:tau_max}.

Let $P_{\tau_{\max}}^{(\nu)}$ be the joint distribution on the $|\Ac| \times \tau_{\max}$ matrix of unquantized rewards:
the $(i,j)$-th entry of this matrix contains the $j$-th unquantized reward for arm $i$ under instance $\nu$. Under the event $\tau \le \tau_{\max}$, the algorithm's output does not depend on any rewards beyond those appearing in this matrix.  In other words, the output $\hat{k}$ is a (possibly randomized) function of this matrix.

By picking $\eta > 0$ to be sufficiently small in Lemma \ref{lem:two_instances}, we can find an instance $\nu' \in \cE$ such that $k_{\nu} \notin \Ac_{c\tilde{\epsilon}}(\nu')$ and
\begin{equation}
    \dTV\big( P_{\tau_{\max}}^{(\nu)}, P_{\tau_{\max}}^{(\nu')} \big) \le \delta.
\end{equation}
Here, $P_{\tau_{\max}}^{(\nu')}$ is defined similarly to $P_{\tau_{\max}}^{(\nu)}$, but for instance $\nu'$.
Since the output $\hat{k}$ is a (possibly randomized) function of the matrix defining $P_{\tau_{\max}}^{(\cdot)}$, we have
 \begin{equation}
    \label{eq: DPI}
     \dTV\big(  \PP_{\nu},  \PP_{\nu'} \big) \le \dTV\big( P_{\tau_{\max}}^{(\nu)}, P_{\tau_{\max}}^{(\nu')} \big) \le \delta
 \end{equation}
 by the data processing inequality for $f$-divergence~\cite[Theorem 7.4]{polyanskiy2024information}.
Using the definition $\dTV(P,Q) = \sup_{A} |P(A) - Q(A)|$, and applying~\eqref{eq: DPI},~\eqref{eq:success_nu},~\eqref{eq: delta very small}, we obtain
\begin{equation}
    \PP_{\nu'}[\hat{k} = k_{\nu} \cap \tau \le \tau_{\max}] \ge 
    \PP_{\nu}[\hat{k} = k_{\nu} \cap \tau \le \tau_{\max}] -
    \dTV\big(  \PP_{\nu},  \PP_{\nu'} \big)  
    \ge
    \frac{1-2\delta}{|\Ac|} - \delta > \delta.
    \label{eq:failure_nu'}
\end{equation}
Since $k_{\nu} \notin \Ac_{c \tilde{\epsilon} }(\nu')$, this means that the algorithm is \emph{not} $(c\tilde{\epsilon}, \delta)$-reliable (see Definition~\ref{def: reliable}), we have arrived at the desired contradiction.
\end{proof}

Corollary~\ref{cor: zero gap is unsolvable} can be proved similarly by using the ``limiting'' version of Lemma~\ref{lem:two_instances} (see Remark~\ref{rem: limit version of two instance lemma}).

\section{Details on Remark~\ref{rem: further improvement} (Improved Gap Definition)}
\label{sec: appendix potential improvement}

\subsection{Modified Arm Gaps}
We first state the modified gap definition explicitly by replacing $Q^+_{(\cdot)}(q - \Delta)$ and $Q_{(\cdot)}(q + \Delta)$ in Definition~\ref{def: our gap}
    with $\max\big\{0, Q^+_{(\cdot)}(q - \Delta)\big\}$ and $\min\big\{\lambda, Q_{(\cdot)}(q + \Delta)\big\}$ respectively, and provide an instance that has a positive modified gap but zero gap under the original definition.
    
\begin{definition}[Modified arm gaps]
\label{def: modified gap}
     Fix an instance $\nu \in \cE$.
     Let $\tilde{\epsilon}$ and $\A_{\epsilon}$ be as in Definition~\ref{def: our gap}.
    For each arm $k \in \A$, we define the improved gap $\tilde{\Delta}_{k} =
    \tilde{\Delta}_{k}(\nu, \lambda, \epsilon, c, q) \in \left[0, \min(q, 1-q) \right]$ as follows: 
    \begin{itemize}
        \item  
        
        If $k \not\in \A_{\epsilon}$, then $\tilde{\Delta}_{k}$ is defined as
        \begin{equation}
            \sup
            \left\{
                \Delta 
                \in \left[0, \min(q, 1-q) \right]
                \colon
               \min\{\lambda, Q_k(q + \Delta)   \}
                \le
                 \max\limits_{a \in \A  }
                 \left\{
                 \max\left\{0,  Q^+_{a}(q - \Delta) \right\}  - \tilde{\epsilon} 
                 \right\}
                \right\}
        \end{equation}

        \item

             If $k \in \A_{\epsilon}$, then we define $\tilde{\Delta}_{k} = \max\limits_{\A_{\epsilon} \subseteq S \subseteq \A}
        \tilde{\Delta}_{k}^{(S)}$, where 
        \begin{equation}
        \label{eq: improved Delta k^S}
            \tilde{\Delta}_{k}^{(S)} =
           \sup
            \Big\{
                \Delta \in 
               \Big[0, \min_{a \not\in S} \tilde{\Delta}_{a}  \Big]
                :
                \max\{0, Q^+_{k}(q - \Delta)\}
                \ge 
                \max\limits_{ a \in S \setminus \{k\}} 
                \min\{\lambda, Q_{a}(q + \Delta)\} - c \tilde{\epsilon}
                \Big\}
        \end{equation}
        for each subset $S$ satisfying $\A_{\epsilon} \subseteq S \subseteq \A$.
                    
    \end{itemize}
We use the convention that the minimum  (resp. maximum) of an empty set is $\infty$ (resp. $- \infty$).
\end{definition}

\begin{remark}[Intuition on the modified arm gap]
    Fix an instance $\nu = (F_k) \in \cE$.
     An interpretation of this modified gap is that
    $\tilde{\Delta}_{k}(\nu, \lambda, \epsilon, c, q) =
    \Delta_{k}(\mathrm{clipped}(\nu), \lambda, \epsilon, c, q)$,
    where $\mathrm{clipped}(\nu) = (\tilde{F}_k) \in \cE$
    is the instance with all distributions supported on $[0, \lambda]$ defined by
    \begin{equation}
        \tilde{F}_k(x)  =
    \begin{cases}
        0 & \text{ for } x < 0 \\
        F_k(x) & \text{ for } 0 \le x < \lambda \\
        1 & \text{ for } x > \lambda 
    \end{cases}
    \quad 
    \text{for each } k \in \A.
    \end{equation}
    That is, $\tilde{F}_k$ is obtained from $F_k$ by moving all mass below 0 to 0, and all mass above $\lambda$ to $\lambda$.  Note that an algorithm could be designed to clip rewards in this way, but our improved upper bound in Theorem \ref{theorem: modified upper bound} below applies even when Algorithm \ref{alg: main} is run without change.
\end{remark}

It is straightforward to verify that the modified gap is at least as large as the unmodified gap (Definition~\ref{def: our gap}), i.e., $\tilde{\Delta} \ge \Delta$. We provide an example of bandit instance that has positive gap under the modified definition but is zero using the unmodified definition. Consider $q = 1/2$, let $\lambda \ge 2 \epsilon > 0$, and consider two arms $\A = \{1, 2\}$ with an identical CDF as follows:
\begin{equation}
    F_1(x) = 
    F_2(x) =
    \begin{cases}
        0 & \text{ for } x < \lambda - \epsilon/3 \\
        0.5 & \text{ for } \lambda - \epsilon/3 \le x < 2 \lambda \\
        1 & \text{ for } x \ge 2 \lambda 
    \end{cases},
\end{equation}
and so both arms are satisfying, i.e., $\A_{\epsilon} = \A$.
Note that for any $\Delta > 0$, we have
\begin{equation}
    Q^+_2(0.5 - \Delta)  =
    \lambda - \epsilon/3 <
    2\lambda - \epsilon \le
     2\lambda - c\tilde{\epsilon} =
    Q_1(0.5 + \Delta) - c\tilde{\epsilon},
\end{equation}
where the second inequality follows from the discussion in~\eqref{eq: tilde eps 1 and 2}--\eqref{eq: c1 tilde eps 1 and 2}. It follows that
\begin{equation}
    \Delta_2 
    = \Delta_{2}^{\A}
    =
    \sup
    \left\{
        \Delta \in [0,0.5]
        :
        Q^+_2(0.5 - \Delta) 
        \ge
        Q_{1}(0.5 + \Delta) - c\tilde{\epsilon}
        \right\} 
    = 0
\end{equation}
under the original gap definition. By symmetry, we also have $\Delta_1 = 0$.
However, under the modified definition, we have
\begin{align}
     \tilde{\Delta}_2 
    = \tilde{\Delta}_{2}^{\A}
    &= \sup
    \left\{
        \Delta \in [0, 0.5]
        :
        \max\{0,  \lambda - \epsilon/2 \}
        \ge
        \min\{\lambda, 2 \lambda\} - c\tilde{\epsilon}
        \right\} \\
    &= \sup
    \left\{
        \Delta \in [0, 0.5]
        :
          \lambda - \epsilon/3 
        \ge
        \lambda  - c\tilde{\epsilon}
        \right\} \\
        &= 0.5,
\end{align}
    where the last inequality follows since 
    $c\tilde{\epsilon} \ge \epsilon/3 $ for any $c \ge 1$ 
    (see the calculation in Remark~\ref{rem: picking large enough c}).

\subsection{Improved Upper Bound}
With the modified gap definition, we obtain the following improved upper bound.

\begin{theorem}[Improved upper bound]
\label{theorem: modified upper bound}
   Fix an instance $\nu \in \cE$, and suppose Algorithm~\ref{alg: main} is run with input $(\A, \lambda, \epsilon, q, \delta)$ and parameter $c \ge 1$.
    Let $\A_{\epsilon}(\nu) $ be as defined in~\eqref{def: performance def} and let the gap $\tilde{\Delta}_{k} = \tilde{\Delta}_{k}(\nu, \lambda, \epsilon, c, q)$ be as defined in Definition~\ref{def: modified gap} 
    for each arm $k \in \A$.
    Under Event~$E$ as defined in Lemma~\ref{lem: good events},
    the total number of arm pulls is upper bounded~by    
    \begin{equation}
        O
        \left(
        \left(
        \sum_{ k \in \A }
        \dfrac{1}{ \max\big( \tilde{\Delta}_{k},  \tilde{\Delta}  \big)^2} \cdot 
        \left( 
         \log \left(\frac{1}{ \delta } \right) +
         \log \left(\frac{1}{ \max\big( \tilde{\Delta}_{k},  \tilde{\Delta}  \big)}\right) +
         \log \left(\frac{c \lambda K}{ \epsilon } \right)    
        \right)
        \right)
        \right),
    \end{equation}
    where $\tilde{\Delta}  =  \tilde{\Delta}(\nu, \lambda, \epsilon, c, q) = \max\limits_{a \in \A_{\epsilon}(\nu)} \tilde{\Delta}_{a}$.
\end{theorem}

The proof is essentially identical to the proof of Theorem~\ref{theorem: upper bound}, but requires tightening of~\eqref{eq: lower approx quantile anytime bound} and~\eqref{eq: upper approx quantile anytime bound}
of anytime quantile bound to
    \begin{equation} 
    \label{eq: modified lower approx quantile anytime bound}
       \max\{0, Q^+_k\big(q -  \Delta^{(t)} \big) \}
        \le \mathrm{LCB}_t(k) + \tilde{\epsilon}
    \end{equation}
    and
    \begin{equation} 
    \label{eq: modified upper approx quantile anytime bound}
        \mathrm{UCB}_t(k) 
        <
         \min\{\lambda, Q_k\big(q + \Delta^{(t)} \big)\} + \tilde{\epsilon}
    \end{equation}
respectively.
Note that the two new bounds \eqref{eq: modified lower approx quantile anytime bound} and~\eqref{eq: modified upper approx quantile anytime bound} can be verified easily using the properties that $ \mathrm{LCB}_t(k) \ge 0$ and $ \mathrm{UCB}_t(k) \le \lambda$ (see Lines~\ref{eq: initiate default conf interval},~\ref{LCB definition}, and~\ref{UCB definition} of Algorithm~\ref{alg: main}), as well as the established bounds~\eqref{eq: lower approx quantile anytime bound} and~\eqref{eq: upper approx quantile anytime bound}.

\subsection{Removing the Assumption in Theorem~\ref{thm: zero gap is unsolvable} (Unsolvability)}
\label{sec: assumption removal}

The assumption involving $\eta_0$ in Theorem~\ref{thm: zero gap is unsolvable} is included to ensure that both $(G_k)^{-1}(q) =  Q_k(q-\eta) $ and $(G_a)^{-1}(q) =  Q_a(q+\eta) $ are in $[0, \lambda]$ in the proof of Lemma~\ref{lem:two_instances}, so that the constructed instance $\nu'$ satisfies $\nu' \in \cE$. As mentioned in Remark~\ref{rem: remove additional assumption}, the assumption can be removed if we use the modified gap instead; formally, we have the following.

 \begin{theorem}[Zero gap is unsolvable -- assumption-free version]
 \label{thm: modified zero gap is unsolvable}
    Let $\lambda, \epsilon, c,$ and $q$ be fixed, 
    and let $\tilde{\epsilon} = \tilde{\epsilon}(\lambda, \epsilon, c)$ be as defined in 
    Algorithm~\ref{alg: main}.
    Let $\tilde{\Delta} = \tilde{\Delta}(\nu, \lambda, \epsilon, c, q)$ be as defined in Theorem \ref{theorem: modified upper bound}.    
    If an instance $\nu \in \cE$ satisfies $\tilde{\Delta} = 0 $, then $\nu$ is $c\tilde{\epsilon}$-unsolvable.
 \end{theorem}

  The proof is essentially identical to the proof of Theorem~\ref{thm: zero gap is unsolvable}, and requires only some straightforward modifications in Lemma~\ref{lem:two_instances}. Specifically, under the new gap definition,~\eqref{eq: arm a positive eta c tilde epsilon} would be replaced by
     \begin{equation}
          \max\{0, Q^+_{k}(q - \eta)\}
       <
        \min\{\lambda, Q_{a}(q + \eta)\} - c \tilde{\epsilon}
        \Big\}
     \end{equation}
    We then construct instance $\nu'$ in a similar manner to the proof of Lemma~\ref{lem:two_instances}, but the definitions of $G_a$ and $G_k$ modified to include clipping:
    \begin{enumerate}[topsep=0pt, itemsep=0pt]
        \item 
        $G_a$ is any distribution obtained by moving $\eta$-probability mass from the interval $(-\infty, Q_a(q))$ to the point $\min\{\lambda, Q_{a}(q + 2\eta)\}$;
    
        \item 
        $G_k$ is any distribution obtained by moving $\eta$-probability mass from the interval $(Q_k(q), \infty)$ to the point $\max\{0, Q_{k}(q - 2\eta)\}$.    
    \end{enumerate}     
    It now follows that
     \begin{equation}
         (G_k)^{-1}(q) = \max\{0, Q_{k}(q - \eta)\} 
         \in [0, \lambda]
     \end{equation}
     and
      \begin{equation}
         (G_a)^{-1}(q) =  \min\{\lambda, Q_{a}(q + \eta)\} \in [0, \lambda],
     \end{equation}
     and hence $\nu' \in \cE$ as desired.

\end{document}